%% file: arXiv_NeurIPS.tex
\documentclass{article}

\PassOptionsToPackage{numbers, compress}{natbib}

\usepackage[preprint]{neurips_2022}

\usepackage[utf8]{inputenc} 
\usepackage[T1]{fontenc}    
\usepackage{url}            
\usepackage{booktabs}       
\usepackage{amsfonts}       
\usepackage{nicefrac}       
\usepackage{microtype}      
\usepackage{xcolor}         

\usepackage{algorithm}
\usepackage{algorithmic}

\usepackage{amsmath, amssymb}
\usepackage{amsthm}
\usepackage{bbm}
\usepackage{mathbbol}
\usepackage{float}
\usepackage{dirtytalk}
\usepackage[mathscr]{euscript}
\usepackage{fancybox}

\usepackage[english]{babel}

\usepackage{booktabs}
\usepackage{empheq}
\usepackage{pifont}
\usepackage{enumitem}
\usepackage[verbose]{wrapfig}
\usepackage{placeins}
\usepackage{bm}

\definecolor{mydarkblue}{rgb}{0,0.08,0.45}
\usepackage[colorlinks=true,
    linkcolor=mydarkblue,
    citecolor=mydarkblue,
    filecolor=mydarkblue,
    urlcolor=mydarkblue,
    pdfview=FitH]{hyperref}

\usepackage{tikz}
\usetikzlibrary{shapes}
\usetikzlibrary{decorations.markings}
\usetikzlibrary{plotmarks}
\usetikzlibrary{patterns}
\usepackage{ctable}
\usepackage{multirow}
\usepackage{multicol}

\usepackage[framemethod=tikz]{mdframed}

\usepackage{color, colortbl}
\definecolor{mydarkblue}{rgb}{0,0.08,0.45}
    
\usepackage{empheq}

\definecolor{myteal}{RGB}{27,158,119}
\definecolor{myorange}{RGB}{217,95,2}
\definecolor{myred}{RGB}{231,41,138}
\definecolor{mypurple}{RGB}{152,78,163}
\definecolor{myblue}{RGB}{55,126,184}
\definecolor{mygreen}{RGB}{0,100,0}

\newtheorem{lemma}{Lemma}

\newtheorem{theorem}{Theorem}

\newtheorem{proposition}{Proposition}
\newtheorem{assumption}{Assumption}

\definecolor{myblue}{HTML}{D2E4FC}

\include{macros}

\newif\ifshowcomments
\showcommentsfalse

\ifshowcomments
\newcommand{\ba}[1]{\textcolor{olive}{[BA: #1]}}
\newcommand{\jp}[1]{\textcolor{blue}{[JP: #1]}}
\else
\newcommand{\jp}[1]{}
\newcommand{\ba}[1]{}
\fi
\ifshowcomments
\title{\textcolor{red}{[NB: COMMENTS DISPLAYED!]\\}Implicit Regularization or Implicit Conditioning? Exact Risk Trajectories of SGD in High Dimensions
}
\else
\title{Implicit Regularization or Implicit Conditioning? Exact Risk Trajectories of SGD in High Dimensions}
\fi

\author{%
  Courtney Paquette\thanks{CP is a CIFAR AI chair, MILA and CP was supported by a Discovery Grant from the
Natural Science and Engineering Research Council (NSERC) of Canada; website \url{https://cypaquette.github.io/}. Research by EP was supported by a Discovery Grant from the
Natural Science and Engineering Research Council (NSERC) of Canada; website \url{https://elliotpaquette.github.io/}. }\\
  Google Research, Brain and McGill University\\
  \texttt{courtney.paquette@mcgill.ca} \\
   \And
   Elliot Paquette \\
  McGill University \\
   \texttt{elliot.paquette@mcgill.ca} \\
   \AND
   Ben Adlam \\
   Google Research, Brain \\
   \And
   Jeffrey Pennington \\
  Google Research, Brain \\
}

\begin{document}

\maketitle

\begin{abstract} 
Stochastic gradient descent (SGD) is a pillar of modern machine learning, serving as the go-to optimization algorithm for a diverse array of problems. While the empirical success of SGD is often attributed to its computational efficiency and favorable generalization behavior, neither effect is well understood and disentangling them remains an open problem. Even in the simple setting of convex quadratic problems, worst-case analyses give an asymptotic convergence rate for SGD that is no better than full-batch gradient descent (GD), and the purported implicit regularization effects of SGD lack a precise explanation. In this work, we study the dynamics of multi-pass SGD on high-dimensional convex quadratics and establish an asymptotic equivalence to a stochastic differential equation, which we call homogenized stochastic gradient descent (HSGD), whose solutions we characterize explicitly in terms of a Volterra integral equation. These results yield precise formulas for the learning and risk trajectories, which reveal a mechanism of implicit conditioning that explains the efficiency of SGD relative to GD. We also prove that the noise from SGD negatively impacts generalization performance, ruling out the possibility of any type of implicit regularization in this context. Finally, we show how to adapt the HSGD formalism to include streaming SGD, which allows us to produce an exact prediction for the excess risk of multi-pass SGD relative to that of streaming SGD (bootstrap risk). 

\end{abstract}

\section{Introduction}
Stochastic gradient descent (SGD) is the algorithm of choice for optimization in modern machine learning and has been hailed as a major reason for deep learning's success~\citep{bottou2012stochastic,Goodfellow-et-al-2016}. Explanations for the effectiveness of SGD typically refer to its computational efficiency and to its favorable generalization properties, but theoretical understanding of these purported benefits is far from complete.

The efficiency of SGD has been the subject of extensive research, dating back to the original work of~\citet{robbins1951} and extending to modern large-scale machine learning applications (see e.g.~\citep{bottou2010large,bottou2007tradeoffs}). However, despite its widespread adoption and algorithmic simplicity, surprisingly little is known about how SGD performs in the types of high-dimensional optimization problems that occur in practice. Part of the challenge in deriving robust high-level conclusions about the efficiency of SGD is simply that those conclusions can depend on precisely which quantities are measured and what assumptions are leveraged. For example, in the extreme setting where the samples are one-hot vectors, running SGD on a quadratic function is actually identical to running full-batch gradient descent; as such, any statements about the two algorithms' relative efficiency must be data-dependent. Furthermore, the majority of prior analyses focus on the streaming or single-pass setting, where each sample is seen a single time. While this setting is appropriate when the number of samples $n$ is much larger than the dimensionality $d$, it does not adequately describe the practically-relevant overparameratized or high-dimensional settings where $d\gtrsim n$.

Moreover, the practical success of SGD has been so remarkable in recent years that a growing body of literature has suggested that its benefit to generalization extends beyond what any improved efficiency might reasonably afford~\citep{smith2021origin,jastrzkebski2017three,chaudhari2018stochastic,shallue2019measuring,smith2017bayesian}. Some of the myriad explanations for SGD's favorable generalization properties include the local geometry of minimizers~\citep{keskar2016on,hochreiter1997flat, WuMaE, gurbuzbalaban2020heavy-tail}, connections to approximate Bayesian inference~\citep{mandt2017stochastic}, and the regularization properties of noise~\citep{smith2020on}, among many others. While some of the these perspectives are intuitive and compelling, they are often difficult to rigorously establish from either an empirical or a theoretical perspective. Empirically, simulations at large scale command significant computational resources, and it can be challenging to push to sufficiently late times or sufficiently large batches to establish the appropriate baselines~\citep{shallue2019measuring,smith2020on}. Theoretically, the strongest existing results are again in the single-pass setting, for which a number of works have established excess risk bounds for quadratic problems~\citep{bach2013non,dieuleveut2017harder,ge2019step,wu2021last}. Much less is known in the multi-pass setting, though stability results were established by~\citep{hardt2016train}, and some recent works have begun examining generalization~\citep{lei2021generalization}.

In this work, we study the dynamics of multi-pass SGD on high-dimensional convex quadratic functions and derive exact asymptotic predictions for the learning and risk trajectories. Our analysis establishes an asymptotic equivalence to a stochastic differential equation, which we call homogenized stochastic gradient descent (HSGD), whose solutions we characterize explicitly in terms of a Volterra integral equation. These results allow us to define a precise data-dependent implicit-conditioning ratio (ICR) that determines whether SGD is more efficient than its full-batch cousins. The ICR favors SGD for many practical datasets, providing some explanation for the observed superior efficiency of SGD; interestingly, we also highlight settings for which SGD is less efficient than full-batch momentum gradient descent, underscoring the data-dependence of the conclusions. Moreover, our results also show that SGD does not improve generalization performance, whether measured in-distribution or out-of-distribution, and therefore that SGD does not offer any form of implicit regularization in this setting. We emphasize that our results do not rule out possible benefits for non-convex problems, but they do provide some of the first explicit negative results in the convex quadratic case. 

\subsection{Contributions}
Our primary contributions are to:
\begin{enumerate}
\item Establish the equivalence of quadratic statistics computed on the iterates of SGD and on a particular stochastic Langevin diffusion process called homogonized SGD (Theorem~\ref{thm:lppp});
\item Exactly characterize the asymptotic training and risk trajectories as the solutions of a deterministic Volterra integral equation (Theorem~\ref{thm:trainrisk});
\item Prove that the noise from SGD negatively impacts generalization performance, both in- and out-of-distribution (Section~\ref{sec:implicit_regularization}), but explain why the impact is often minimal in practice;
\item Introduce the implicit-conditioning ratio that describes when and by how much SGD accelerates convergence relative to the best full-batch methods (Section~\ref{sec:implicit_conditioning});
\item Analyze the limit of streaming SGD to show its inability to capture many salient features of the dynamics of multi-pass SGD (Section~\ref{sec:streaming_main}).
\end{enumerate}

\section{Preliminaries and background}

\paragraph{Problem setting.} We consider high-dimensional $\ell^2$-regularized least squares problems defined by,
 \begin{equation}\label{eq:rr}
    \min_{\xx\in\mathbb{R}^d} \Big\{ f(\xx) \defas  \frac{1}{2}\|\AA \xx - \bb\|_2^2 + \frac{\delta}{2} \|\xx\|^2 = \sum_{i=1}^n \underbrace{\frac{1}{2} \left ( (\aa_i \xx - b_i)^2  + \frac{\delta}{n} \|\xx\|^2 \right)}_{\defas f_i(\xx)} \Big\},
\end{equation}
where $\delta \ge 0$ is the ridge-regularization parameter. We denote the ridgeless empirical risk as
\begin{equation} \label{eq:lsq}
    \mathscr{L}(\xx) \defas \frac{1}{2} \|\AA \xx-\bb\|_2^2.
\end{equation}
On the problem \eqref{eq:rr}, the steps taken by gradient decent (GD) can be written recursively as
\begin{equation}\label{eq:gd}
    \mgd_{k+1} = \mgd_k - \gamma_k \nabla f(\mgd_k) 
    = \mgd_k - \gamma_k \AA^T ( \AA \mgd_k - \bb) - {\gamma_k \delta} \mgd_k + \Delta(\mgd_k - \mgd_{k-1}),
\end{equation}
where $\Delta > 0$ is the momentum parameter, $\gamma_k$ is the learning rate schedule, and $\xx_0 \in \mathbb{R}^d$ is an initial vector assumed to be independent of all other randomness and having norm at most $1$. When $A$ is large, computing these updates can be expensive, so an unbiased estimator for the true gradient is often used, where a subset of the data points are selected uniformly at random. We focus on the setting with batch size equal to one and without momentum, which we refer to as stochastic gradient descent (SGD), and for which the iterates can be written recursively as
\begin{equation} \begin{aligned} \label{eq:sgd}
    \sgd_{k+1} = \sgd_k - \gamma_k \nabla f_{i_k}(\sgd_k) = \sgd_k - \gamma_k \AA^T \ee_{i_k} \ee_{i_k}^T ( \AA \sgd_k - \bb) - \tfrac{\gamma_k \delta}{n} \sgd_k,
\end{aligned} \end{equation}
where the $i_k \sim\text{Unif}([n])$ iid. While it would also be possible to consider mini-batch SGD, previous work has shown that batch sizes that are vanishingly small as a fraction of the number of samples are equivalent to the single-batch analysis, after appropriately adjusting the time by a factor of the batch size \citep[Theorem 1]{paquette2020halting}; similarly, we do not consider high-dimensional SGD with momentum as it degenerates to SGD \citep{paquette2021dynamics}.  See also \cite{jain2018accelerating}.

\paragraph{Diffusion approximations and homogenized SGD.}
A common paradigm for understanding SGD is through stochastic Langevin diffusions (SLD), i.e. solutions of equations of the form
\begin{equation}\label{eq:SLD}
\dif \XX_t = -\gamma \nabla f(\XX_t) + \sqrt{ \SSigma_t } d\BB_t\,,
\end{equation}
where $\gamma$ is the step size of SGD, $f$ is the loss function, $\BB_t$ is a $d$-dimensional standard Brownian motion, and the matrix $0 \preceq \SSigma_t \in \mathbb{R}^{d \times d}$ models the noise covariance. 
In many analyses, no concrete connection between SGD and SLD is developed, and the diffusion is merely used to build intuition. A common example is the isotropic case ($\SSigma_t \propto \II_d$), for which the Fokker-Planck equation implies that the dynamics are reversible with respect to a density proportional to $e^{-f(x)}$. Consequently, the process can escape local minima, exhibiting a trade-off between the entropy and depth of minima and thereby highlighting a possible mechanism of implicit regularization. In the general anisotropic case, describing the stationary distribution is more difficult; nonetheless, the local geometry near minima of $f$ can be analyzed, see \citep{chaudhari2018stochastic,kunin2021rethinking}. 

While this type of implicit entropic regularization might ultimately underlie the generalization benefits of SGD for nonconvex problems, currently we lack a precise connection between a concrete SLD and a practical nonconvex learning problem. As such, the implicit regularization effects of SGD on nonconvex losses remains a largely unsolved problem.

For convex quadratics, however, the implications of Eq.~\eqref{eq:SLD} are quite clear: there is no notion of implicit regularization as the noise in SLD negatively impacts generalization performance. Note that because the noise is mean zero, any SLD is centered around \emph{gradient flow} (GF) $\gf_t$, which solves
\begin{equation}
\label{eqn:GF}
\dif \gf_t = -\nabla f(\gf_t)\,,
\end{equation}
leading to the following conclusion for generalization:
\begin{lemma}\label{lemma:impossibility}
Suppose the objective function is $f(\xx) = \frac12\bigl(\|\AA \xx-\bb\|^2 + \delta \|\xx\|^2\bigr)$.
Suppose $(\XX_t : t\in[0,\infty))$ is an SLD (i.e.\ $\XX_t$ solves \eqref{eq:SLD}) with $\|\SSigma_t\|_{op}$  almost surely bounded by some $C < \infty$.
Suppose the population risk $\mathcal{R} : \mathbb{R}^d \to \mathbb{R}$ is a convex function and denote $\xx_*\defas \lim_{t\to\infty}\gf_t$, then
\[\underbrace{\Exp [\mathcal{R}(\XX_t)]}_{\text{pop. risk of SLD}} \geq \underbrace{\mathcal{R}(\gf_{\gamma t} )}_{\substack{\text{pop. risk of}\\\text{gradient flow}}} \quad \text{for all $t \ge 0$ and hence,} \quad 
\underbrace{\liminf_{t \to \infty} \Exp [\mathcal{R}(\XX_t)]}_{\substack{\text{limiting pop. risk of SLD}}} \geq \underbrace{\mathcal{R}(\xx_*)}_{\substack{\text{limiting pop. risk}\\\text{gradient flow}}}.
\]
If in addition $\mathcal{R}$ is strictly convex, and $\SSigma_t\to \SSigma_\infty$ with $\SSigma_\infty \succ 0,$ then the inequality is strict.
\end{lemma}
\begin{proof}
The mean $\Exp [\XX_t]$, by the linearity of the gradient $\nabla f$, is GF.
Under the conditions given, the law of $\XX_t$ converges to a Gaussian variable centered at $\xx_*$.  Hence by Fatou's lemma and Jensen's inequality, the inequality follows.
\end{proof}
We emphasize that this conclusion applies even under general distribution shifts, so long as the risk remains a convex function. Still, the utility of Lemma~\ref{lemma:impossibility} may not be immediately clear, as it pertains to SLD and we have not yet established any concrete connection between SLD and the process of interest, SGD. Nor is it evident what form such a connection should take---the agreement between SGD and an SLD cannot occur at the level of individual states since the randomness from each process is not assumed to be coupled. Instead, the most we can hope for is that statistics of the processes agree. Specifically, we might hope that matching the noise structure of SGD with a careful choice of SLD will cause relevant statistics, like the population risk, to be equal.

It turns out that such a choice of SLD exists for convex quadratic problems in high dimensions, and is given by \textit{homogenized SGD} (HSGD), introduced simultaneously in \citep{mori2021logarithmic,paquette2021dynamics}. Both the empirical and population risks ($\mathcal{L}$, $\mathcal{R}$ resp.) of HSGD agree with the same of SGD in the high-dimensional limit (see Thm.~\ref{thm:lppp}). 
Mathematically, HSGD is the strong solution of the stochastic differential equation: 
\begin{equation}\label{eqF:HSQD}
    \dif \XX_t \defas -\gamma(t) \nabla \mathcal{L}(\XX_t) \dif t + \gamma(t) \sqrt{\tfrac{2}{n} \mathcal{L}(\XX_t) \nabla^2 \mathcal{L}(\XX_t) } \dif \BB_t,
    \quad\text{for quadratic } \mathcal{L},
\end{equation}
where again $\BB_t$ is a $d$-dimensional standard Brownian motion, $\gamma(t)$ is the learning rate schedule, and the initial condition is $\XX_0 = \xx_0$. Roughly, HSGD is a diffusion approximation to SGD that gains explanatory power when the \emph{dimensionality} is large. It does not require the step size $\gamma$ to be small. Note that as with other universality results, the details of the noise distribution are not relevant and only the second-order correlations contribute, which are carefully matched by HSGD to SGD.


The precise sense of the comparison requires us to evaluate low-dimensional statistics of the high-dimensional dynamics; ``low-dimensional'' must be effective, in that the  univariate statistics of the SGD iterates concentrate around the same statistic evaluated on HSGD. For understanding generalization or implicit regularization properties, a important statistic is the population risk, $\mathcal{R}$. 

\paragraph{Assumptions.} For all parts of our analysis to hold, the pair $(\AA, \bb)$ of the data matrix $\AA \in \mathbb{R}^{n \times d}$ and target vector $\bb \in \mathbb{R}^n$ must satisfy some \textit{quasi-random} assumptions---a set of deterministic conditions on the pair $(\AA, \bb)$ that are satisfied with high probability by natural classes of random matrix-vector pairs (see Appendix \ref{sec:quasi_random} for specifics).  We use the convention that the target and initialization vectors are bounded independent of $n$,  $\|\bb\|^2_2 \le C$ and $\|\xx_0\|_2^2 \le C$, respectively.  

We illustrate some examples below that we have shown to satisfy the quasi-random assumptions.
\begin{itemize}
    \item Gaussian linear regression.  Here the rows of $\AA$ are iid and drawn from a Gaussian with norm-bounded covariance $\Sigma$ and the target $\bb$ is drawn from a generative model, $\bb = \AA \widetilde{\xx} + \eeta$ for some unknown signal $\widetilde{\xx} \in \mathbb{R}^d$ and independent noise $\eeta \in \mathbb{R}^n$.
    \item Subgaussian linear designs.  In the example above, we can relax the Gaussian assumption to be of the form $\xx \Sigma^{1/2}$ for $\xx$ a vector of iid centered subgaussian random variables~\cite{zou2022risk,jain2018accelerating}.
    \item Gaussian random features with a linear ground truth~\citep{mei2019generalization,adlam2020understanding,Rahimi2008Random,adlam2022random}. Suppose $\AA$ is given by $\sigma( \XX \WW)$ for an iid standard Gaussian weight matrix $\WW$ and Gaussian data matrix $\XX$.  Suitable assumptions on the activation function $\sigma$ and the covariance $\Sigma$ of $\XX$ added.
\end{itemize}

\begin{assumption}\label{ass:risk}
The population risk $\mathcal{R} : \R^d \to \R$ is a quadratic, that is, it is a degree-2 polynomial or, equivalently, can be represented by 
\[\mathcal{R}(\xx) = \frac{1}{2} \xx^T \TT \xx + \uu^T \xx + c\]
for some $d \times d$ symmetric matrix $\TT$, vector $\uu \in \mathbb{R}^d$, and scalar $c \in \mathbb{R}$. We further assume that $\|\nabla^2 \mathcal{R}\|_{op} \leq C$, $\|\nabla \mathcal{R}(0)\|_2^2 \leq C$, and |$\mathcal{R}(0)| \leq C$.
\end{assumption}

A natural population risk is given by $\mathcal{R}(\xx) = \tfrac{1}{2} \EE[ (\aa \cdot \xx - b)^2]$ where $(\aa, b) \sim \mathcal{D}$. This distribution $\mathcal{D}$ may or may not be the same as the distribution that generated the data $[\AA ~|~ \bb]$ used in training. 

\ba{Something about high-dimensional limit?} \jp{Yes we ought to perhaps describe the setting, though not sure we need a technical assumption as I think/hope the conditions are embedded in the actual theorem statements?}

As we work in the high-dimensional limit, we suppose that $\gamma_k=\gamma(k/n)$ for a smooth, bounded function $\gamma(\cdot)$ such that $\gamma(t) \to \gamma \in [0, \infty)$ and $\widehat{\gamma} \defas \sup_{t \ge 0} \gamma(t) < \infty$.

\section{Main results}
\label{sec:main_results}
\begin{figure}[t]
    \centering
    \includegraphics[scale = 0.35]{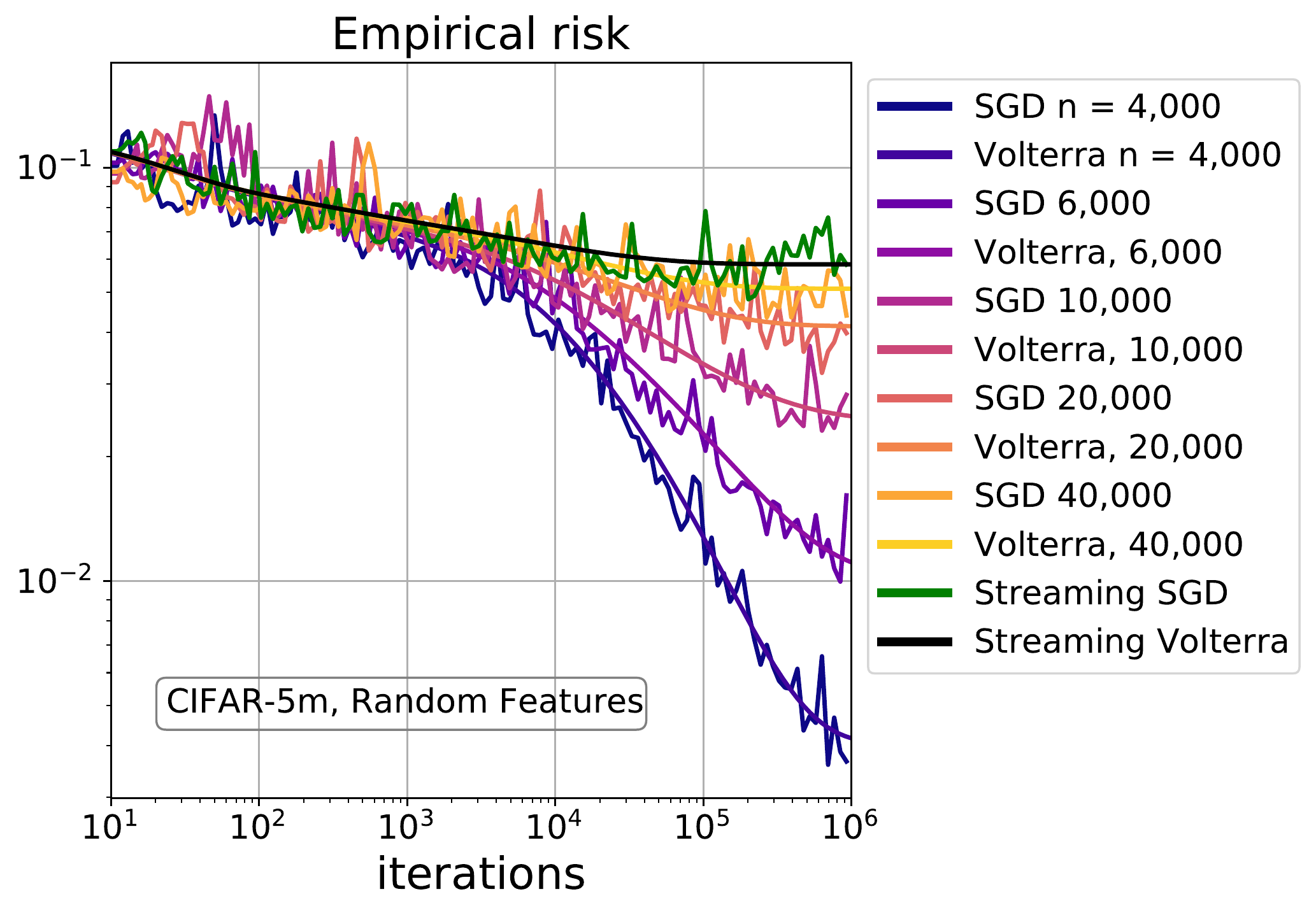}
     \includegraphics[scale =0.35]{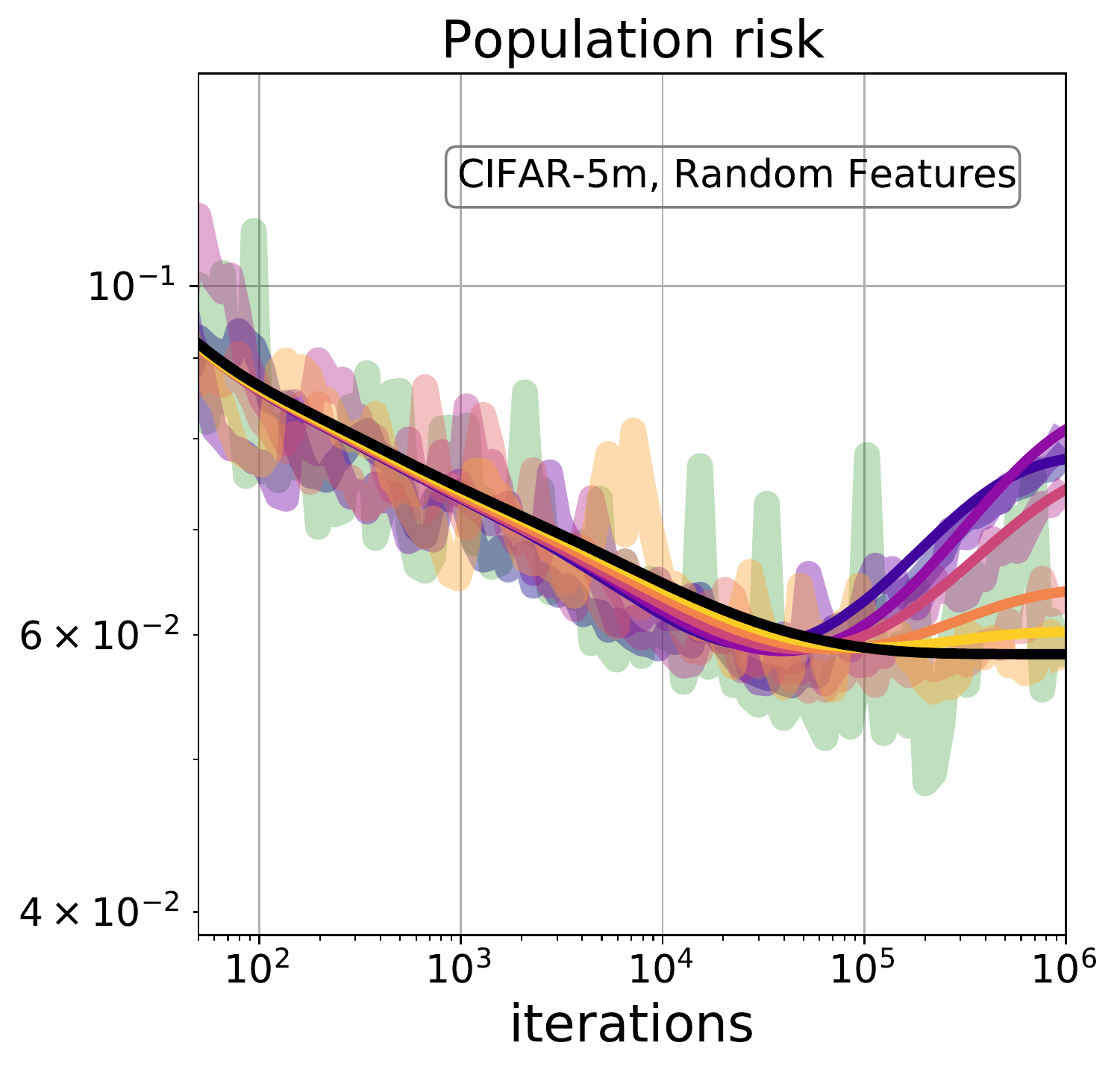}
    \caption{\textbf{Single runs of SGD vs.\ HSGD (Volterra) in streaming} on standarized CIFAR-5M \citep{nakkiran2021bootstrap} with car/plane class vector ($1,000,000$ samples); a standarized ReLu \eqref{eq:standarized_relu} random features model (see Appendix~\ref{sec: motivating_applications}) was applied with increasing number of samples $n$ and fixed $d=6000$. The predicted behavior from HSGD (denoted by Volterra) matches the performance of single runs of SGD for finite $n$ and streaming $(n = \infty)$. Shaded region (right) is the moving average of a single run of SGD. Empirical risk (left) increases monotonically with $n$ to its limit while population risk generally decreases with $n$.  Streaming corresponds to $n=\infty$ (see Sec.~\ref{sec:streaming_main}).
    For consistency across sample sizes, time is measured in iterations. Additional details in App~\ref{sec:numerical_simulations}.}
    \label{fig:CIFAR_5M_streaming}
\end{figure}
Our main results are analyzable (non-asymptotic) expressions for the empirical risk $\mathcal{L}$ and the population risk $\mathcal{R}$ of SGD at any time $t$ for the high-dimensional least squares problem~\eqref{eq:rr}. To begin, we first establish the following equivalence between SGD and HGSD.

\begin{theorem}[Equivalence of SGD and HSGD]\label{thm:lppp}
Suppose the pair $(\AA, \bb) \in \mathbb{R}^{n \times d} \times \mathbb{R}^d$ satisfy the quasi-random assumptions with $  d^\varepsilon \le n \le d^{1/\varepsilon} $ for some $\varepsilon \in (0,1]$. Let the iterates $\xx_t = \sgd_{\lfloor t\rfloor}$ be generated from multi-pass SGD Eq.~\eqref{eq:sgd} and $\XX_t$ be the solution of Eq.~\eqref{eqF:HSQD}.  
Then for any deterministic $T > 0$ and any $D>0,$ there is a $C>0$ such that
\[
\Pr \biggl[
\sup_{0 \leq t \leq T} \left \|
\begin{pmatrix}
\mathscr{L}( \xx_{\lfloor tn \rfloor }  )\\
\mathcal{R}(\xx_{\lfloor tn \rfloor }  ) 
\end{pmatrix} -  \begin{pmatrix}
\mathscr{L}(\XX_t)\\
\mathcal{R}(\XX_t)
\end{pmatrix} \right \|_2 > d^{-\varepsilon/2}
\biggr]
\leq Cd^{-D}.
\]
\end{theorem}
As a result, for the rest of this paper, we can use homogenized SGD to analyze the behavior of multi-pass SGD. We will establish a similar conjecture for streaming SGD in Section~\ref{sec:streaming_main}. 



While the comparison of SGD to HSGD requires relatively strong assumptions on $\AA$ and $\bb$, the analysis of HSGD can be performed under weaker assumptions (no quasirandomness assumptions are needed). It suffices to suppose the problem is high dimensional in the following sense:
\begin{assumption}\label{ass:train}
The empirical risk $\mathscr{L}$ satisfies $\tr \nabla^2 \mathscr{L}=n$ and $0 \preceq \nabla^2 \mathscr{L} \preceq n d^{-\epsilon}$ for some $\epsilon > 0.$ 
\end{assumption}
This corresponds to the normalization where $\nabla^2 \mathscr{L} = \AA^T \AA$ and each row of $\AA$ is length $1$ and hence $\tr \nabla^2 \mathscr{L}=n$. \jp{Could be useful to have such a remark but I'm not seeing this normalization stated previously -- are we missing it or in Sec. 2 or...?}


Under Assumptions \ref{ass:risk} and \ref{ass:train}, the dynamics of the empirical and population risk under HSGD concentrate around a deterministic dynamical system driven by a Volterra integral equation: 

\global\mdfdefinestyle{exampledefault}{%
outerlinewidth=0pt,innerlinewidth=0pt,
roundcorner=5pt,backgroundcolor=myblue,
leftmargin=0.2cm,rightmargin=0.2cm
}

\newpage
\begin{mdframed}[style=exampledefault]
{\textbf{Volterra Dynamics, Multi-pass.} The following deterministic dynamical system is the high-dimensional limit for $\mathscr{L}(\XX_t)$ and $\mathcal{R}(\XX_t)$, respectively}
\begin{align} 
&\Psi_t
=
\mathscr{L}\bigl( 
\gf_{\Gamma(t)}\bigr)
+
\int_0^t 
K(t,s; \nabla^2 \mathscr{L}) 
\Psi_s
\dif s & \text{(Empirical risk)} \label{eqa:VLoss} \\ 
&\Omega_t
=
\mathcal{R}\bigl( 
\gf_{\Gamma(t)}\bigr)
+\int_0^t 
K(t,s; \nabla^2 \mathcal{R}) 
\Psi_s
\dif s & \text{(Population risk)} \label{eqa:PLoss}
\end{align}

    where the \textit{integrated learning rate $\Gamma$} and \textit{kernel $K$}, for any $d \times d$ matrix $\PP$, respectively are 
{\small
\begin{equation} \begin{aligned} \label{eqa:V} 
 \Gamma(t)\! = \!  \int_0^t \!  \gamma(s)\,\dif s,
\, 
\, 
K(t,s ; \PP) \!= \!
\tfrac{\gamma^2(s)}{n}
\tr
\bigl(
(\nabla^2 \mathscr{L})
\PP
\exp\bigl( -2(\nabla^2 \mathscr{L} + \delta \II_d)( \Gamma(t) - \Gamma(s))\bigr)
\bigr).
\end{aligned}
\end{equation} }
\end{mdframed}

\vspace{0.25cm}

\begin{theorem}[Concentration of HSGD around Volterra dynamics]\label{thm:trainrisk}
Under Assumptions \ref{ass:risk} and \ref{ass:train}, for any $T > 0$ and for any $D>0$ there exists sufficiently large $C>0$ such that for all $d>0$
\[
\Pr\biggl[
\sup_{0 \leq t \leq T}\biggl\|
\begin{pmatrix} \mathscr{L}(\XX_t) \\ \mathcal{R}(\XX_t)\end{pmatrix}
-
\begin{pmatrix}\Psi_t \\ \Omega_t \end{pmatrix}
\biggr\| > d^{-\epsilon/2} 
\biggr] \leq Cd^{-D},
\]
where $\Psi_t$ and $\Omega_t$ solve \eqref{eqa:VLoss} and \eqref{eqa:PLoss}.
\end{theorem}
We give a formal proof of the concentration result in Appendix \ref{sec:VolterraConcentration} in Theorem \ref{thm:conc}.

\subsection{No implicit regularization from SGD}
\label{sec:implicit_regularization}
From \eqref{eqa:PLoss}, for convex $\mathcal{R}$ we observe immediately that the population risk $\Omega_t$ is only larger than the population risk of GF. 
Moreover, we have an explicit formula for the excess risk due to SGD noise,
\[
\underbrace{\Omega_t-\mathcal{R}\bigl( 
\gf_{\Gamma(t)}\bigr)}_{\text{excess risk due to SGD}} 
\defas \int_0^t 
K(t,s; \nabla^2 \mathcal{R}) 
\times \! \! \! \! \! \! \! \!\!\underbrace{\Psi_s}_{\text{limiting loss $\mathcal{L}$}} \! \! \! \! \! \! \! \!
\dif s.
\]
Note that the population risk of SGD tracks that of GF.  If GF overfits, SGD overfits as well; there is no statistical regularization due to the noise of SGD applied to empirical risk minimization (ERM).

We can further analyze the long-time behavior of SGD with exact limiting values for this excess risk.
\begin{theorem}[Time infinity risk values]\label{thm:eventualrisk}
If $\gamma(t) \to 0$ as $t\to\infty$ but $\Gamma(t) \to \infty$ (i.e.\ the usual Robbins-Monro setting), then the excess population risk of SGD over GF tends to $0$.
If on the other hand $\gamma(t) \to \gamma \in (0, 2(\tfrac{1}{n} \tr \big \{ \tfrac{(\AA^T\AA)^2}{\AA^T \AA + \delta \II_d} \big \})^{-1}$ , then with $\Psi_\infty$ given by the limiting empirical risk,
\[
\Psi_\infty 
=
\mathscr{L}\bigl( 
\gf_{\infty}\bigr)
\times
\biggl(
1
-
\frac{\gamma}{2n}
\tr
\biggl\{
\frac{
(\nabla^2 \mathscr{L})^2
}
{
\nabla^2 \mathscr{L} + \delta \II_d
}
\biggr\}
\biggr)^{-1}
\]
the excess risk due to SGD converges to
\[
\Omega_t-\mathcal{R}\bigl( 
\gf_{\Gamma(t)}\bigr)
\to \frac{\gamma  \Psi_\infty}{2n}  \times 
\tr
\biggl\{
\frac{
(\nabla^2 \mathcal{R})
(\nabla^2 \mathscr{L})
}
{
\nabla^2 \mathscr{L} + \delta \II_d
}
\biggr\}.
\]

\end{theorem}
There are a few conclusions to draw directly from this.  In the interpolation regime, that is where $\Psi_\infty = 0,$ there is no excess risk due to SGD and there is no need to send $\gamma$ to $0$.  Moreover, if the empirical risk $\Psi_\infty$ is small, the excess risk due to SGD is proportional to $\gamma \Psi_\infty$, and hence it is frequently orders of magnitude smaller than other potential sources of error.  Furthermore, the excess risk is affected by how similar the population and empirical risks are, in the large directions. Ridge regularization can substantially reduce excess risk due to SGD in cases where population risk has many small eigenvalues. In summary, either by sending $\gamma \to 0,$ working in the interpolation regime, or otherwise in a regime $\Psi_\infty$ is small, \emph{the excess risk incurred by running SGD is minimal}.

\subsection{Implicit conditioning of SGD}
\label{sec:implicit_conditioning}

In contrast, the algorithmic advantages of SGD are substantial.  To simplify the discussion, we consider only the case of constant learning rate $\gamma$.  In this case, the kernel in \eqref{eqa:VLoss} and \eqref{eqa:PLoss} simplifies to a convolution kernel, which has a much simpler theory.
To characterize the rates, we define 
\(
\lambda_{\min}
\)
as the smallest non-zero eigenvalue of $\nabla^2\mathscr{L}$.  Then for generic initial conditions, (in particular almost surely if $\XX_0$ is nonzero isotropic), GF has the following convergence rate
\[
\lim_{t \to \infty} \biggl(
\mathscr{L}(\gf_{\gamma t})
-
\mathscr{L}(\gf_\infty)
\biggr)^{1/t} = 
\begin{cases}
e^{-\gamma(\lambda_{\min}(\nabla^2 \mathscr{L}) + \delta)}, & \text{if } \delta > 0, \\
e^{-2\gamma \lambda_{\min}(\nabla^2 \mathscr{L})}, & \text{otherwise}.
\end{cases}
\]
Here we use the notation that $\lambda_{\min}(\HH)$ and $\lambda_{\max}(\HH)$ are the smallest and largest eigenvalues of the matrix $\HH$. The rate of convergence of $\Psi_t$ to $\Psi_\infty$ can be no faster than the underlying GF, given by the rate above.  On the other hand, for larger $\gamma$ the Volterra term in \eqref{eqa:VLoss} can frustrate the convergence.  The \emph{Malthusian exponent} of the convolution Volterra equation is given by
{\small \begin{equation}\label{eq:Malthusian}
    \lambda_* \!= \!
    \inf\biggl\{
    x : 
    1
    = \! \!\int_0^\infty \! \!\! \! \!e^{ xt}K(t; \nabla^2\mathscr{L})\,\dif t 
    \defas \gamma^2 \!\!
    \int_0^\infty \! \!\! \!\!
    e^{ xt}
    \tr
    \bigl(
    \bigl(\nabla^2\mathscr{L}\bigr)^2
    \exp\bigl( -2\gamma(\nabla^2\mathscr{L}+ \delta \II_d )t\bigr)
    \bigr)\,\dif t 
    \biggr\}.
\end{equation}}
As $\nabla^2\mathscr{L}$ is finite dimensional, we have that $\lambda_* \leq 2\gamma(\lambda_{\min}(\nabla^2 \mathscr{L}) + \delta),$ owing to the divergence of the integral as $x$ approaches this value from below.  Note that in principal the Malthusian exponent can be negative, in which case SGD is \emph{divergent}. The Malthusian exponent gives the effective rate of convergence of constant learning rate SGD.  Define
\begin{equation}\label{eq:ratefunction}
\Xi(\gamma)
\defas
\begin{cases}
\min\{
\gamma(\lambda_{\min}(\nabla^2 \mathscr{L})+\delta),
\lambda_*(\gamma)
\}
& \text{ if } \delta > 0, \\
\lambda_*(\gamma)
& \text{ if } \delta = 0.
\end{cases}
\end{equation}

\begin{theorem}[SGD convergence rates, average-case]\label{thm:sgdrates}
Then the rates of convergence of both the empirical and population risk are controlled by this parameter
\[
\lim_{t \to \infty} \bigl(
\Psi_t
-
\Psi_\infty
\bigr)^{1/t} = e^{-\Xi(\gamma)}
=
\lim_{t \to \infty} \bigl(
\Omega_t
-
\Omega_\infty
\bigr)^{1/t}.
\]
Furthermore, when $\gamma = n(\tr(\AA^T\AA))^{-1}$, we have the rate guarantee $\Xi(\gamma) \geq \tfrac{\lambda_{\min}(\nabla^2 \mathscr{L}) +\delta}{2}.$
\end{theorem}

\begin{figure}[t]
  \begin{minipage}[c]{0.57\textwidth}
    \includegraphics[width=\textwidth]{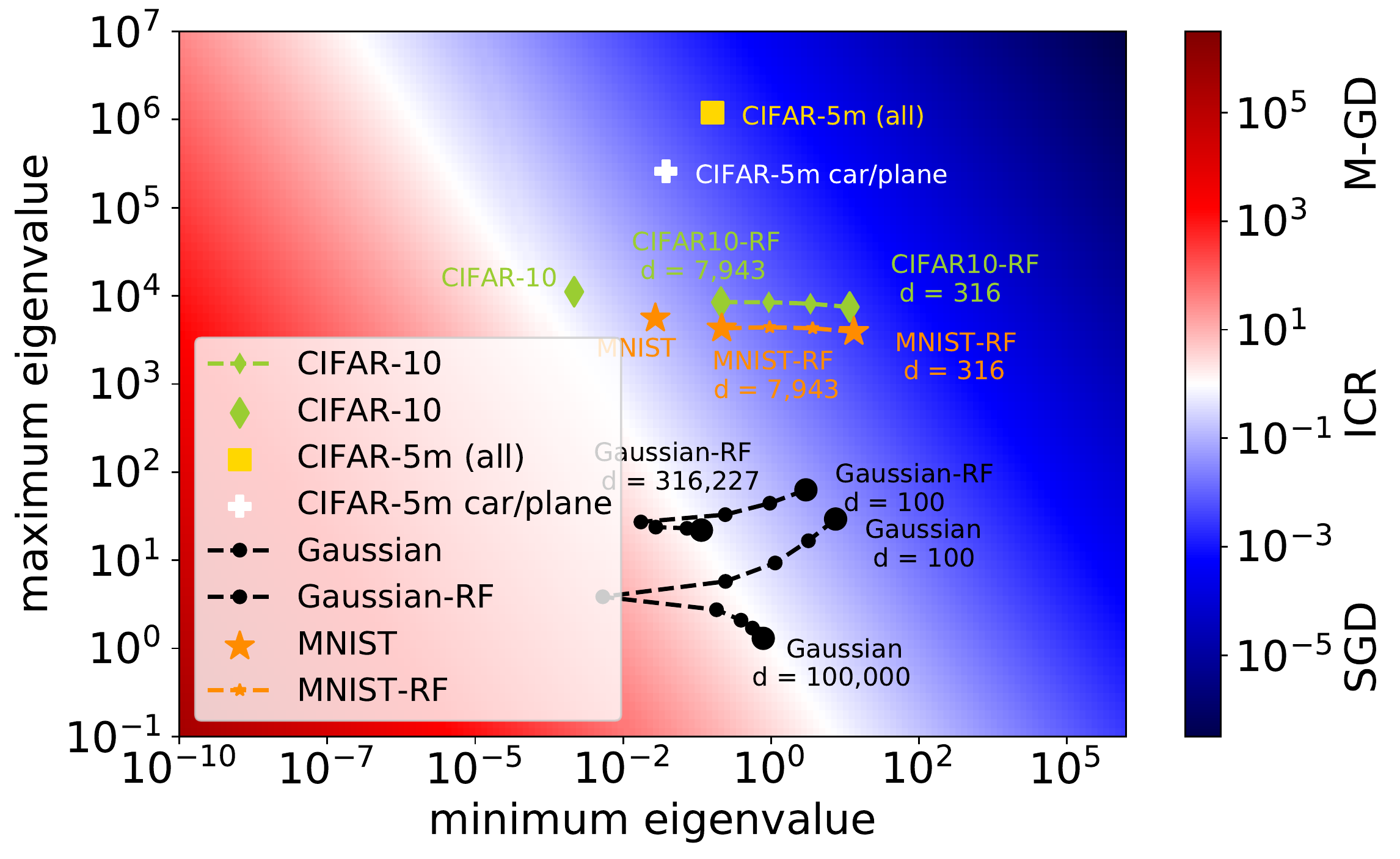}
  \end{minipage}\hfill
  \begin{minipage}[c]{0.41\textwidth}
    \caption{ICR as function of largest and smallest eigenvalues of $\AA$ with trace normalized to be $1$; blue (smaller, SGD favored) and red (larger, full batched M-GD favored). Points indicate ICR for image datasets (MNIST, CIFAR-10, CIFAR-5m) as well as their images under ReLu random feature maps of various dimensions (lines). Gaussian random features ($n=2000$, $n_0= 100$, various $d$) and Gaussian data ($n = 2000$, various $d$) shows ICR for over- and under-parameterized models. Natural datasets tend to favor SGD.} \label{fig:ICR}
  \end{minipage}
\end{figure}

The major difference between SGD and full batch methods such as momentum gradient descent (M-GD; see Appendix \ref{sec:momentum} for definitions) is that they have different sensitivities to the Hessian spectrum of the empirical risk $\mathscr{L}(\xx)=\frac12\|\AA\xx-\bb\|^2$.  Define the condition numbers
\[
\kappa \defas
\frac
{\lambda_{\max}(\nabla^2\mathscr{L}) + \delta}
{\lambda_{\min}(\nabla^2\mathscr{L}) + \delta}
\quad
\text{and}
\quad
\overline{\kappa}
\defas
\frac
{\tfrac{1}{n}\tr(\nabla^2\mathscr{L})}
{\lambda_{\min}(\nabla^2\mathscr{L}) + \delta}.
\]
The first of these is the classical condition number of the ridge problem, while the second is the averaged condition number that regulates the behavior of SGD in the high-dimensional limit.
M-GD has been long established to have a rate of convergence, with proper tuning, controlled by the square root of the condition number \citep{Polyak1962Some}, which is known to be optimal amongst first order algorithms.
\begin{theorem}[Convergence rates for M-GD]\label{thm:mgdrates}
For isotropic random initialization $\xx_0$ or noisy $\bb$, $\delta > 0$, and strictly convex population risk $\mathcal{R}$ 
\[
\bigl(\mathscr{L}(\mgd_k)-\mathscr{L}(x_*)\bigr)^{1/k}
\Asto[k] 
\biggl(
\frac{\sqrt{\kappa}-1}
{\sqrt{\kappa}+1}
\biggr)
\quad\text{and}
\quad
\bigl(\mathcal{R}(\mgd_k)-\mathcal{R}(x_*)\bigr)^{1/k}
\Asto[k] 
\biggl(
\frac{\sqrt{\kappa}-1}
{\sqrt{\kappa}+1}
\biggr).
\]
\end{theorem}
\noindent See Appendix \ref{sec:momentum} for elaboration. 

In light of Theorems \ref{thm:sgdrates} and \ref{thm:mgdrates}, we can define the \emph{implicit-conditioning ratio} as
\[
\operatorname{ICR}
\defas
\frac{\overline{\kappa}}{\sqrt{\kappa}}
\approx
\log
\biggl(
\frac{\sqrt{\kappa}-1}
{\sqrt{\kappa}+1}
\biggr)
\overline{\kappa},
\]
which measures the efficiency of SGD over M-GD in that SGD with constant learning rate $n/\tr(\nabla^2\mathscr{L})$ trains in an ICR-multiple of the number of epochs that M-GD requires (lower is better for SGD).

Problems favor SGD when there are large outlier eigenvalues, a common feature of Hessian spectra in practice~\citep{sagun2016eigenvalues,sagun2017empirical,abdel2007adaptive}. Indeed, if the largest eigenvalues are on the same order as the \emph{unnormalized trace}, individual SGD iterates are as effective as full-batch gradient.  In contrast, when the Hessian spectrum is tightly packed, which is less common in practice but can occur after some preprocessing techniques or e.g. for uncorrelated Gaussian samples, then M-GD is favored. See Fig.~\ref{fig:ICR}.

\section{Streaming SGD} \label{sec:streaming_main}
In this section we introduce constant learning rate ($\gamma(t) \equiv \gamma$) streaming SGD. Let
$(\aa_k, b_k)_{k=1}^\infty$ be iid samples from a $\R^d \times \R$-dimensional distribution $\mathcal{D}$. 
Streaming SGD using data from $\mathcal{D}$, which we denote $\mathcal{D}$-SGD, is the algorithm 
\begin{equation} \begin{aligned} \label{eq:opsgd}
    \ss_{k+1} = \ss_k
    -\gamma \aa_k (\aa_k \cdot \ss_k - b_k) 
    \quad\text{for}\quad k \in \{0,1,2,\dots\}.
\end{aligned} \end{equation}
This naturally describes one-pass SGD, in which data points are used only once. If $\mathcal{R}$ is the expected risk (i.e. population risk), and if $\mathcal{R}$ is given by $\frac{1}{2} \EE (\aa \cdot \xx - b)^2$ with $(\aa, b) \sim \mathcal{D}$, then $\mathcal{D}$-SGD is directly solving the population risk minimization. This is an idealized situation as one does not have access to infinite data in practice.

\begin{figure}
    \centering
    \includegraphics[scale = 0.35]{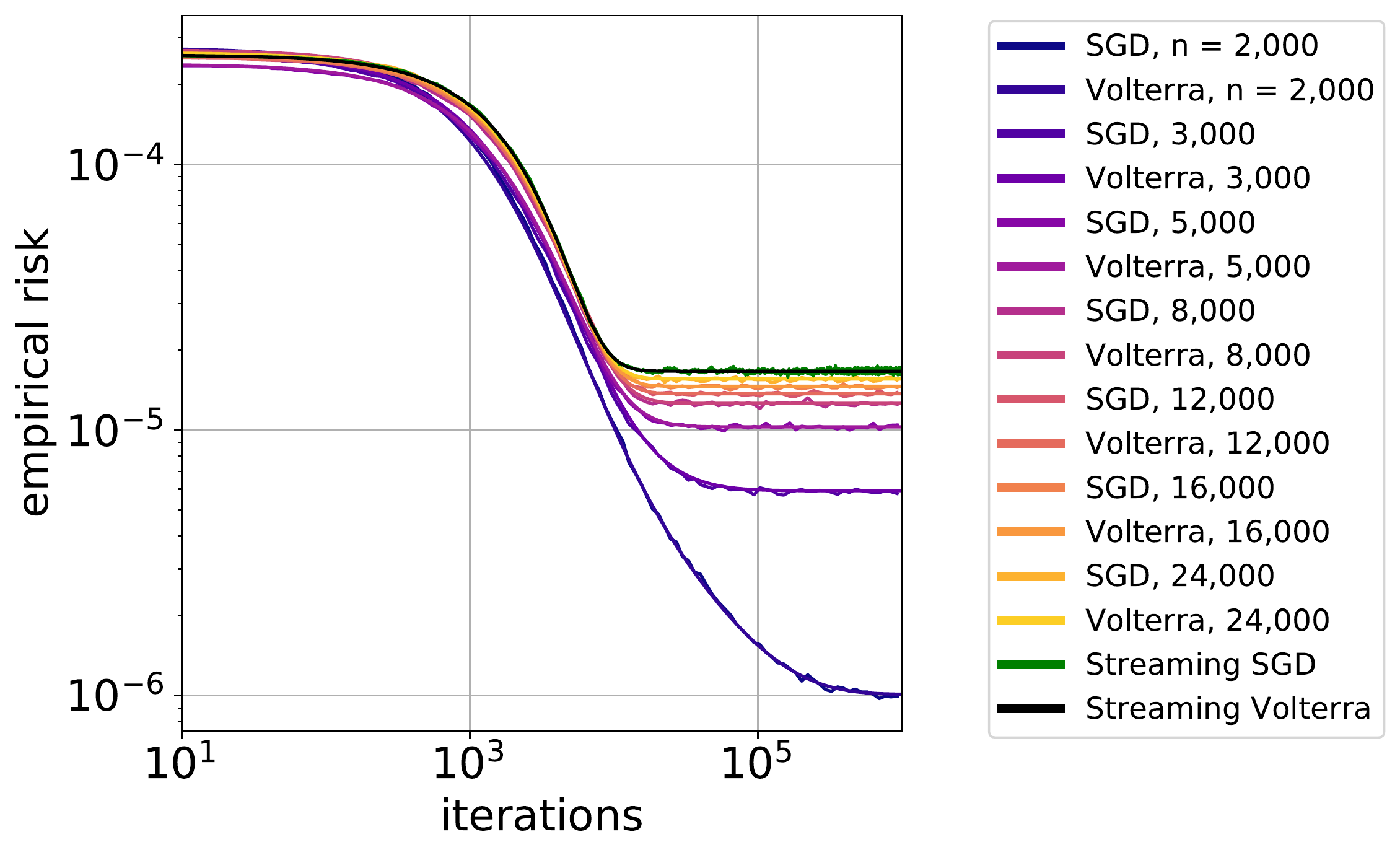}
     \includegraphics[scale =0.35]{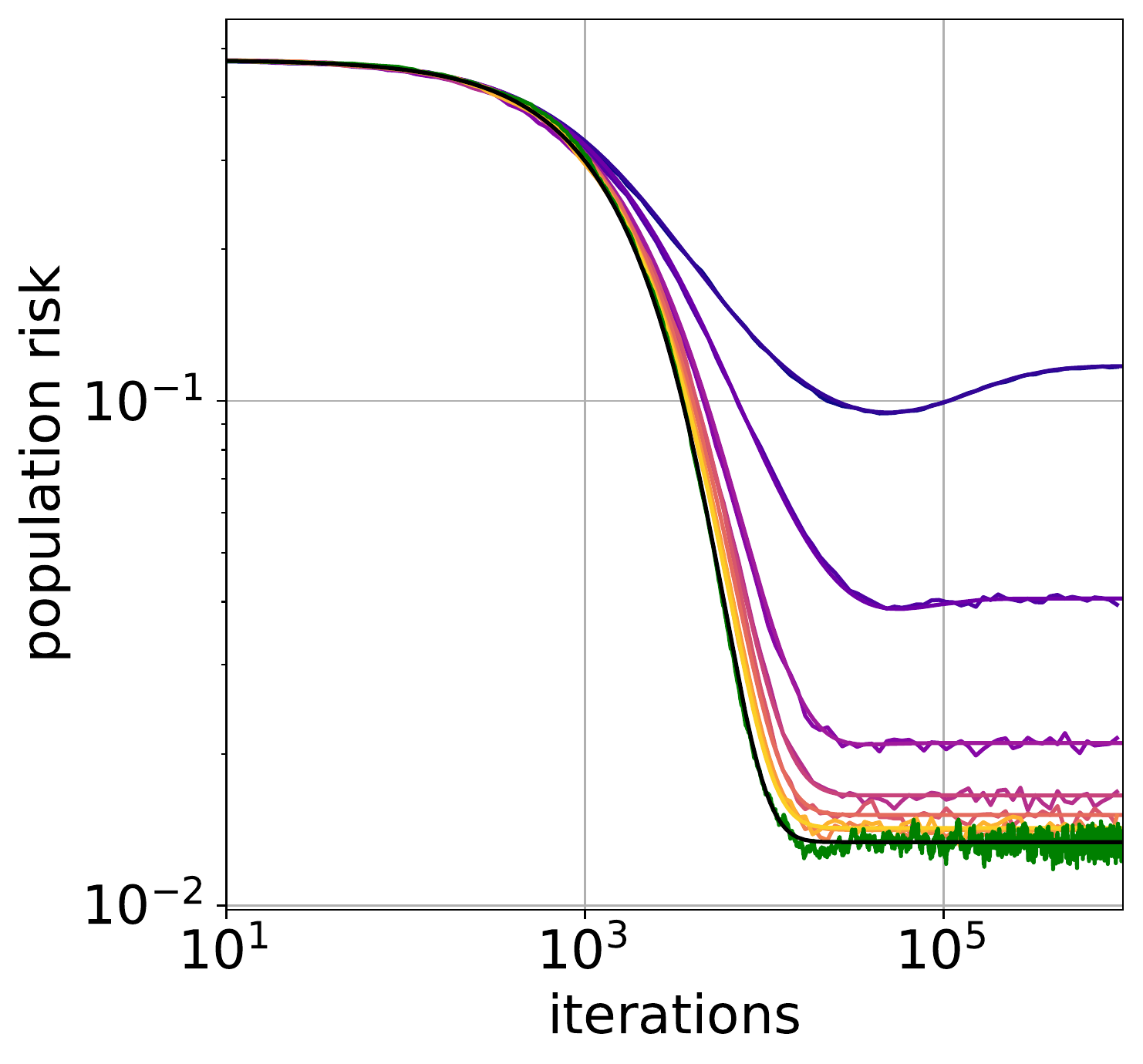}
    \caption{\textbf{Single runs of SGD vs.\ HSGD (Volterra) for a simple Gaussian linear regression problem}, with increasing number of samples $n$ and $d=2000$.  Empirical risk (left) increases monotonically with $n$ to a limit while population risk (right) decreases monotonically in $n$.  Streaming corresponds to $n=\infty$.  Covariance of Gaussian samples is $\II_d$, with target given by $b=a \cdot \beta + \eta Z$ for $\eta=0.2$ and $\beta$ a unit vector and $Z\sim N(0,1)$.  For consistency across sample sizes, time is measured in iterations.}
    \label{fig:gaussianstreaming}
\end{figure}

\paragraph{Deterministic behavior of streaming risks and comparison to HSGD.}

$\mathcal{D}$-SGD can encompass multi-pass SGD by letting $(\aa_k, b_k)_{k=1}^n$ be the first $n$ iid samples from $\mathcal{D}$ and considering  $\widehat{\mathcal{D}}$-SGD for
\begin{equation}\label{eq_dhat}
\widehat{\mathcal{D}}_n
\defas
\frac{1}{n}\sum_{k=1}^n \delta_{(\aa_k, b_k)}.
\end{equation}
This means that a sample from $\widehat{\mathcal{D}}_n$ (\emph{conditionally} on the dataset) is distributed like $(\aa_{I}, b_{I})$, where $I$ is an uniformly random choice of index.  

To enable a comparison of streaming SGD to multi-pass SGD, we suppose for some $n$ the matrix $[\AA \, |\, \bb]$ is $n \times (d+1)$ matrix whose rows are iid samples from $\mathcal{D}$. We construct the empirical risk $\mathcal{L}$ from these $n$ samples as in \eqref{eq:lsq}. We then define the streaming loss
\[
\mathscr{S}(\xx) \defas 
\tfrac{1}{2}\Exp (\aa \cdot \xx - b)^2
=
\tfrac{1}{2n} \Exp \| \AA \xx - \bb\|^2
=
\tfrac{1}{n}
\Exp \mathscr{L}(\xx),
\quad\text{where}\quad (\aa,b) \sim \mathcal{D}.
\]
The associated homogenized SGD representation, which we call homogenized $\mathcal{D}$-SGD, is
\begin{equation}\label{eqa:VSLD}
\dif \YY_t \defas
-\gamma(\nabla {\mathscr{S}}(\YY_t)) \dif t
+ \gamma\sqrt{2\mathscr{S}(\YY_t) (\nabla^2 \mathscr{S}) }\dif \BB_t
\quad\text{for}\quad t \geq 0.
\end{equation}

This naturally leads to Volterra dynamics in which $\nabla^2 \mathscr{L}$ is replaced by $\nabla^2 \mathscr{S}$ in Eqs.~\eqref{eqa:VLoss}-\eqref{eqa:V} whose solution we denote with $\Psi_t^s$ and $\Omega_t^s$ (see Appendix \ref{sec:sld_concentration}).

We prove a weak equivalence between HSGD and homogenized $\mathcal{D}$-SGD in the following theorem. 
\begin{theorem}\label{thm:wkstreaming} Suppose $\XX_0 = \YY_0$. After rescaling time, 
\[\lim_{n \to \infty} \XX_{t/n} = \YY_t \quad \text{uniformly on compact sets of time a.s.}\]
Furthermore, GF converges, $\gf_{\gamma t/n} \to \sgf_{\gamma t}$, and the Volterra equations converge, $\Psi_{t/n} \to \Psi_t^s$ and $\Omega_{t/n} \to \Omega_t^s$, uniformly on compact sets of time.
\end{theorem}

We expect this to hold in much greater generality as suggested by Figures~\ref{fig:CIFAR_5M_streaming} and \ref{fig:gaussianstreaming}. This is an immediate consequence of the law of large numbers due to which $\tfrac{1}{n} \mathcal{L}$ converges almost surely to $\mathcal{S}$; the applications in which we are typically interested would take time large as a function of $n$ (the numerical results are extremely strong, see Figures \ref{fig:gaussianstreaming} and \ref{fig:CIFAR_5M_streaming}), and we leave a deeper mathematical investigation of this point as an open question.  We also note that for $t$ very large with $n$, there is likely another behavior that takes hold.  For example, when $t = n$, observe that $\XX_{t/n}$ will have used approximately $ (1-1/e)n$ samples, whereas $\YY_n$ will have used $n$, and hence we expect a breakdown in the connection. Other works also examined the expected risk of streaming including \citep{bordelon2022learning,ziyin2022strength} but not as the limit of the dynamics of multi-pass SGD as we have done.


\paragraph{Features of generalization.}

While there are connections between streaming and multi-pass SGD, certain behavior is only accessible in the multi-pass setting. For example, early stopping can be a useful ingredient in avoiding overfitting when learning overparameterized models. However, late-time overfitting is only observable with multi-pass SGD and does not occur for streaming SGD (see the nonmonotonicity in population risk in Fig.~\ref{fig:gaussianstreaming}). 

Similarly, comparing streaming and multi-pass SGD has been suggested as a method for analyzing generalization. The bootstrap risk was introduced in \citep{nakkiran2021bootstrap} as an intepretable component in a decomposition of the population risk. It is defined as the excess risk of SGD for ERM $\mathscr{L}$ with $n$ iid samples from $\mathcal{D}$ when compared to $\mathcal{D}$-SGD, i.e.
\begin{equation}\label{eq:Boot}
    \underbrace{
    \Omega_{t/n}
    -\Omega_t^{\text{s}}
    }_{\text{bootstrap risk}} 
    =
    \underbrace{
    \mathcal{R}\bigl( 
    \gf_{\gamma t/n}\bigr)
    -
    \mathcal{R}\bigl( 
   \sgf_{\gamma t }\bigr)
    }_{\text{model risk}} 
    +
    \underbrace{
    \int_0^{t/n}
    K(t/n,u; \nabla^2 \mathcal{R}) 
    \Psi_u
    \dif u
    -
    \int_0^t 
    K^{\text{s}}(t,u; \nabla^2 \mathcal{R}) 
    \Psi^{\text{s}}_u
    \dif u
    }_{\text{SGD bootstrap risk}}.
\end{equation}
Thus the Volterra equations can be used to give an exact expression for the bootstrap risk. In particular, we can predict the iteration at which the bootstrap risk becomes large as the streaming and multi-pass risks bifurcate. This allows for quantitative prediction of the stopping time in Claim 1 of \cite{nakkiran2021bootstrap}, a central conjecture of their paper.

On taking time to infinity, we can further evaluate the SGD bootstrap risk.  
\begin{theorem}\label{thm:bootstraprisk}
If $\gamma \le \min\{2 ( \frac{1}{n} \tr(\AA^T\AA))^{-1}, 2 (\tr \nabla^2 \mathscr{S})^{-1}\}$, then the limiting bootstrap risk is given by
\[
\lim_{t \to \infty}
\Omega_{t/n}
-\Omega_t^{\text{s}}
=
\mathcal{R}\bigl( 
\gf_{\infty}\bigr)
-
\mathcal{R}\bigl( 
\sgf_{\infty}\bigr)
+
\frac{ \gamma \Psi_\infty }{2n}
\times
\tr
\biggl\{
\frac{
(\nabla^2 \mathcal{R})
(\nabla^2 \mathscr{L})
}
{
\nabla^2 \mathscr{L} + \delta \II_d
}
\biggr\}
\]
Here $\Psi_\infty$ is the limiting training losses given by Theorem \ref{thm:eventualrisk}, i.e.
\[
\begin{aligned}
\Psi_\infty 
&=
\mathscr{L}\bigl( 
\gf_{\infty}\bigr)
\times
\biggl(
1
-
\frac{\gamma}{2n}
\tr
\biggl\{
\frac{
(\nabla^2 \mathscr{L})^2
}
{
\nabla^2 \mathscr{L} + \delta \II_d
}
\biggr\}
\biggr)^{-1}.
\end{aligned}
\]
\end{theorem}

\section{Conclusion.}

Using a specific type of SLD (called HSGD) that matches the second-order correlations in the noise of SGD, we demonstrated that their empirical and population risks match in the high-dimensional limit. Moreover, the risks of HSGD behavior deterministically, as described by a Volterra equation. With this connection, we investigated the benefits of SGD on a convex objective. While there is no statistical benefit to generalization from the noise of SGD, in overparameterized, interpolating settings little is lost compared to GD. Moreover, when computational restrictions are imposed, SGD can be radically faster than GD because of its dependence on a different condition number of the Hessian. We characterized this speed up using the ICR, which when calculated for datasets common in deep learning clearly favors SGD. This should highlight the difficulty in studying implicit regularization for SGD empirically: any experiment necessarily has a finite computational budget and may find lower population risks with SGD simply via its improved conditioning. Finally, we demonstrated limitations in using streaming SGD alone as a tool for studying generalization.

As future work, a proper theoretical connection between streaming and multi-pass SGD remains to be made.  A major outstanding problem (both theoretically and empirically) is extending the analysis above to non-quadratic losses, both train and test, and especially to other high-dimensional problems not in the kernel regime. Finally, data augmentation can naturally be considered by randomly augmenting each sample from $\widehat{\mathcal{D}}_n$ in Eq.~\eqref{eq_dhat}.

\bibliographystyle{plainnat}
\bibliography{references}

\newpage

\appendix

\begin{center}
\LARGE{Implicit Regularization or Implicit Conditioning?\\ Exact Risk Trajectories of SGD in High Dimensions}\\
\vspace{0.5em}\Large{Supplementary material \vspace{0.5em}}

\end{center}
The appendix is organized into five sections as follows:
\begin{enumerate}
\item Appendix~\ref{sec:quasi_random} expands upon the assumptions/setting around \eqref{eq:rr} and it discusses some motivating applications such as in- and out- of distribution expected risk and random features.  Moreover, we discuss the equivalence of homogenized SGD and SGD, see Theorem~\ref{thm:lppp}.
    \item Appendix~\ref{sec:sld_concentration} introduces a general Volterra class of equations, called the Volterra SLD class, that encompasses homogenized SGD and its Volterra dynamics in the multi-pass setting (Section~\ref{sec:main_results}) and streaming (Section~\ref{sec:streaming_main}). This general Volterra class allows for more types of additive noise. We prove in this section that the Volterra SLD class concentrates around its mean; thereby deriving the proof of Theorem~\ref{thm:trainrisk}.
    \item We prove in Appendix~\ref{sec:limit_excess_risk} the limiting risk values for the Volterra SLD class (Theorem~\ref{thm:sgdlimit} (constant learning rate) and Theorem~\ref{thm:eventualrisk_SLD} (time dependent learning rate)). These two theorems immediately imply the limiting risk values for homogenized SGD in the multi-pass and streaming settings, see Theorems~\ref{thm:eventualrisk} and \ref{thm:bootstraprisk} respectively. 
    \item Appendix~\ref{sec:app_algorithmic} discusses the exact asymptotic convergence rates for SGD and full batch momentum algorithms on high-dimensional $\ell^2$-regularized least squares problems. The results in this section (e.g.,  Theorems~\ref{thm:sgdrates} and \ref{thm:mgdrates}) were shown in a series of papers \citep{paquette2021dynamics, paquetteSGD2021,paquette2020halting} that explored exact trajectories of loss function. 
    \item Appendix~\ref{sec:numerical_simulations} contains details on the simulations. 
\end{enumerate}


\section{Quasi-random assumptions on the data matrix, targets, and initialization} \label{sec:quasi_random}

The data matrix $\AA \in \mathbb{R}^{n \times d}$, target $\bb \in \mathbb{R}^n$, and initialization $\xx_0 \in \mathbb{R}^d$ may be deterministic or random; we formulate our theorems for deterministic matrix $\AA$ and vectors $\bb$ and $\xx_0$ satisfying various assumptions, and in the applications of these theorems to statistical settings, we shall show that random $\AA$ and $\bb$ satisfy those assumptions. These assumptions are motivated by ERM and, in particular, when the augmented matrix $[\AA \, | \, \bb]$ has rows that are independent and sampled from some common distribution. We call these assumptions \textit{quasi-random}. 

As the problem \eqref{eq:rr} is homogeneous, we adopt the following normalization convention without loss of generality. 

\begin{assumption}[Data-target normalization] \label{assumption:Target} 
There is a constant $C>0$ independent of $d$ and $n$ such that
the spectral norm of $\AA$ is bounded by $C$
and
the target vector $\bb \in \mathbb{R}^n$ is normalized so that $\|\bb\|^2 \leq C$.
\end{assumption}

More importantly, we also assume that the data and targets resemble typical unstructured high-dimensional random matrices. One of the principal qualitative properties of high-dimensional random matrices is the \emph{delocalization of their eigenvectors}, which refers to the statistical similarity of the eigenvectors to uniform random elements from the Euclidean sphere. The precise mathematical description of this assumption is most easily given in terms of resolvent bounds. The resolvent $R(z; \MM)$ of a matrix $\MM \in \mathbb{R}^{d \times d}$ is
\[
  R(z; \MM) =  (z\II_d-\MM)^{-1} \quad \text{for $z \in \mathbb{C}$.}
\]
In terms of the resolvent, we suppose the following:
\begin{assumption}\label{ass: laundry_list}
  Suppose $\Omega$ is the contour enclosing $[0,1+\|\AA\|^2]$
  at distance $1/2$.
  Suppose there is a $\theta \in (0,\tfrac 12)$ for which
  \begin{enumerate}
    \item
      \( \displaystyle
	\max_{z \in \Omega} \max_{1 \leq i \leq n} |\ee_i^T R(z; \AA\AA^T) \bb| \leq n^{\theta-1/2}.
      \)
    \item
      \( \displaystyle
	\max_{z \in \Omega} \max_{1 \leq i \neq j \leq n} |\ee_i^T R(z; \AA\AA^T) \ee_j^T| \leq n^{\theta-1/2}.
      \)
    \item
      \( \displaystyle
      \max_{z \in \Omega} \max_{1 \leq i \leq n} |\ee_i^T R(z; \AA\AA^T) \ee_i - \tfrac 1n\tr R(z; \AA\AA^T)| \leq n^{\theta-1/2}.
      \)
  \end{enumerate}
\end{assumption}
\noindent Only the resolvent of $\AA \AA^T$ appears in these assumptions, and so in effect we are only assuming statistical properties on the left singular-vectors of $\AA$.  This assumption reflects the common formulation of ERM in which the rows of $\AA$ are independent, and so the left singular-vectors of $\AA$ are expected to be delocalized (under some mildness assumptions on the distributions of the rows). The first condition, which involves the interaction between $\AA\AA^T$ and $\bb$, can be understood as requiring that $\bb$ is not too strongly aligned with the left singular-vectors of $\AA$. The other two conditions can be viewed as corollaries of delocalization of the left singular-vectors.

As for the initialization $\xx_0$, we need to suppose that it, like $\bb$, does not interact too strongly with the left singular-vectors of $\AA^T\AA$.  In the spirit of Assumption \ref{ass: laundry_list}, it suffices to assume the following:
\begin{assumption}\label{assumption:init}
Let $\Omega$ be the same contour as in Assumption \ref{ass: laundry_list} and let $\theta \in (0,\tfrac 12)$.  Then
      \[
	\displaystyle \max_{z \in \Omega} \max_{1 \leq i \leq d} |\ee_i^T R(z; \AA^T\AA) \xx_0| \leq n^{\theta-1/2}.
      \]
\end{assumption}
\noindent Note that, as a simple but common case, this assumption is surely satisfied for $\xx_0 = \bm{0}$.  In principle, this assumption is general enough to allow for $\xx_0$ that are correlated with $\AA$ in a nontrivial way, but we do not have an application for such an initialization.  For a large class of nonzero initializations independent from $(\AA,\bb)$, this assumption is satisfied, as a corollary of Assumption \ref{ass: laundry_list}:
\begin{lemma}\label{lem:xo}
    Suppose that Assumption \ref{ass: laundry_list} holds with some $\theta_0 \in (0,\tfrac12)$ and that $\xx_0$ is chosen randomly, independent of $(\AA,\bb)$, and with independent coordinates in such a way that for some $C$ independent of $d$ or $n$
    \[
    \|\Exp \xx_0\|_\infty \leq C/n
    \quad\text{and}\quad
    \max_i\|(\xx_0-\Exp \xx_0)_i\|^2_{\psi_2} \leq Cn^{2\theta_0-1}.
    \]
    For any $\theta >\theta_0$, Assumption \ref{assumption:init} holds with any $\theta > \theta_0$ on an event of probability tending to $1$ as $n \to \infty$.
\end{lemma}
\noindent Note that this assumption allows for deterministic $\xx_0$ having maximum norm $\mathcal{O}(1/n)$, as well as iid centered subgaussian vectors of Euclidean norm $\mathcal{O}(1)$.

To execute the mathematical comparison between SGD and HSGD, we require an additional assumption on the quadratic in the same spirit as Assumption~\ref{assumption:Target}. 

\begin{assumption}[Quadratic statistics] \label{assumption: quadratics} Suppose $\mathcal{R} : \mathbb{R}^d \to \mathbb{R}$ is quadratic, i.e. there is a symmetric matrix $\TT \in \mathbb{R}^{d \times d}$, a vector $\uu \in \mathbb{R}^d$, and a constant $c \in \mathbb{R}$ so that
\begin{equation} \label{eq:statistic}
    \mathcal{R}(\xx_t) = \tfrac{1}{2} \xx_t^T \TT \xx_t + \uu^T \xx_t + c. 
\end{equation}
We also assume that $\mathcal{R}$ satisfies Assumption~\ref{ass:risk}. Moreover, we assume the following (for the same $\Omega$ and $\theta$) as in Assumption \ref{ass: laundry_list}:
\begin{equation} \label{eq:key_lemma_ass}
    	\max_{z,y \in \Omega} \max_{1 \leq i \leq n} 
    	|\ee_i^T \AA \widehat\TT  \AA^T \ee_i
    	-\tfrac1n\tr(\AA \widehat\TT  \AA^T)
    	| \leq \|\TT\|_{\text{op}} n^{-\epsilon}
    	\, \, \text{where} \, \, 
    	\left\{
    	\begin{aligned}
    	&\widehat\TT =  R(z) \TT R(y) + R(y) \TT R(z), \\
    	&R(z) = R(z; \AA^T\AA)
    	\end{aligned}
    	\right.
\end{equation}
\end{assumption}

\noindent This assumption ensures that quadratic $\mathcal{R}$ has a Hessian that is not too correlated with any of the left singular vectors of $\AA$.  Establishing Assumption \ref{assumption: quadratics} can be nontrivial in the cases when the quadratic has complicated dependence on $\AA$.  In simple cases, (especially for the case of the empirical risk and the norm) it follows automatically from Assumption \ref{ass: laundry_list}.
\begin{lemma}\label{lem:qs}
    Suppose that 
    \(
    \mathcal{R}
    \)
    satisfies \eqref{eq:statistic} with $\TT$ given by a polynomial $p$ in $\AA^T \AA$ (especially  $\II$ and the monomial $\AA^T\AA$) having bounded coefficients, and suppose $\uu$ and $c$ are norm bounded independently of $n$ or $d$.  Then supposing Assumptions \ref{assumption:Target} and \ref{ass: laundry_list} for some $\theta_0 \in (0,\tfrac 12)$, for all $n$ sufficiently large and for any $\theta > \theta_0$, Assumption \ref{assumption: quadratics} holds.
\end{lemma}

For proofs of Lemma~\ref{lem:xo} and \ref{lem:qs}, see Section 2 in \citep{paquette2022homogenization}.

\subsection{Motivating applications} \label{sec: motivating_applications}

\paragraph{Training loss and sample covariance matrices.} 
One important (nonstatistical) quadratic statistic, which allows analysis of the optimization aspects of SGD in high dimensions, is the $\ell^2$-regularized loss function $f$ in \eqref{eq:rr}. Then provided that $\AA,\bb$ satisfy Assumptions \ref{assumption:Target} and \ref{ass: laundry_list}, $\xx_0$ is iid subgaussian, Lemmas \ref{lem:xo} and  \ref{lem:qs} and Theorems~\ref{thm:lppp} and \ref{thm:trainrisk} show that $f(\xx_k)$ concentrates around the solution of a Volterra integral equation.
A natural setup under which Assumptions \ref{assumption:Target} and \ref{ass: laundry_list} are satisfied is the following:

\begin{assumption}\label{ass:sc}
Suppose $M > 0$ is a constant.
Suppose that $\SSigma$ is a positive semi-definite $d \times d$ matrix with $\tr \SSigma = 1$ and $\|\SSigma\|_{\text{op}} \leq M / \sqrt{d} < \infty.$
Suppose that $\AA$ is a random matrix $\AA = \ZZ \sqrt{\SSigma}$ where $\ZZ$ is an $n \times d$ matrix of independent, mean $0$, variance $1$ entries with subgaussian norm at most $M < \infty$, and suppose $n \leq M d$.  Finally suppose that $\bb = \AA \bbeta + \xxi$ for $\bbeta,\xxi$ iid centered subgaussian satisfying $\|\bbeta\|^2 =R$ and $\|\xxi\|^2 = \frac{n}{d} \widetilde{R}$.
\end{assumption}

\noindent These assumptions naturally lead to random matrices that satisfy Assumption \ref{ass:sc} with good probability:
\begin{lemma}\label{lem:sc}
    If $(\AA,\bb)$ satisfy Assumption \ref{ass:sc}, then $(\AA,\bb)$ satisfies Assumptions \ref{assumption:Target} and \ref{ass: laundry_list} with probability tending to $1 - e^{-\Omega(d)}$.
\end{lemma}

\noindent Hence, under these assumptions, we conclude:
\begin{theorem}\label{thm:Ab}
Suppose $(\AA, \bb)$ satisfy Assumption \ref{ass:sc}, $\delta >0$ and $\xx_0$ is iid centered subgaussian with $\Exp \|\xx_0\|^2 = \widehat{R}.$ Then for some $\epsilon > 0$, for all $T >0$, and for all $D>0$ there is a $C>0$ such that
\[
\Pr
\biggl(
\sup_{0 \leq t \leq T}
\left\| 
\begin{pmatrix}
\mathscr{L}(\sgd_{\lfloor tn \rfloor}) \\
\tfrac{1}{2} \|\sgd_{\lfloor tn \rfloor} - \bbeta \|^2_2
\end{pmatrix} - 
\begin{pmatrix}
\Psi_t  \\ 
\Omega_t
\end{pmatrix}
\right\|_2 > d^{-\epsilon}
\biggr)
\leq Cd^{-D},
\]
where $\Psi_t$ solves \eqref{eqa:V} and $\Omega_t$ solves \eqref{eqa:PLoss} with $\mathcal{R} = \frac{1}{2}\|\cdot - \bbeta\|^2_2$.
\end{theorem}
\noindent We discuss generalization implications in the the next section.

Theorem \ref{thm:Ab} generalizes \citep{paquetteSGD2021} in that it allows for varying training rates, adds a regularization parameter, and allows for non-orthogonally-invariant designs $\AA$.  We further note that under the assumptions of Theorem \ref{thm:Ab}, we can further approximate the behavior of GF to show that
\begin{equation}\label{eq:fullPsi_training}
\begin{aligned}
\Psi_t &=  \mathscr{L}(\bm{\mathscr{X}}_{\Gamma(t)}^{\text{gf}}) + \frac{1}{n} \int_0^t \gamma^2(s) \tr \bigg ( (\AA^T \AA)^2 e^{-2(\AA^T \AA + \delta \II_d)(\Gamma(t)-\Gamma(s))} \bigg ) \Psi_s \, \dif s\\
   \text{where} \qquad  \mathscr{L}(\bm{\mathscr{X}}_{\Gamma(t)}^{\text{gf}})
    &=
    \frac{R}{2d} \tr \bigg [ (\AA^T\AA) \bigg ( \AA^T \AA (\AA^T \AA + \delta \II_d)^{-1} \big (\II_d-e^{-(\AA^T\AA + \delta \II_d)\Gamma(t)} \big )- \II_d \bigg )^2 \bigg ] \\
    &+
   \frac{ \widetilde{R}}{2d} \tr \bigg [ \bigg ( \AA(\AA^T \AA + \delta \II_d)^{-1}\big [ \II_d - e^{-(\AA^T \AA + \delta \II_d)\Gamma(t)} \big ]\AA^T - \II_n \bigg )^2 \bigg ] \\
    &+
    \frac{\widehat{R}}{2d} \tr \big ( \AA^T \AA e^{-2 (\AA^T \AA + \delta \II_d) \Gamma(t)} \big ).
\end{aligned}
\end{equation}
For the risk $\mathcal{R}(\cdot) = 1/2 \|\cdot - \bbeta\|^2_2$, we have following expression
\begin{equation}\label{eq:fullOmega_training}
\begin{aligned}
    \Omega_t &= \mathcal{R}(\bm{\mathscr{X}}_{\Gamma(t)}^{\text{gf}}) + \frac{1}{n} \int_0^t \gamma^2(s) \tr \bigg ( (\AA^T \AA) e^{-2(\AA^T \AA + \delta \II_d)(\Gamma(t)-\Gamma(s))} \bigg ) \Psi_s \, \dif s\\
    \text{where} \qquad \mathcal{R}(\bm{\mathscr{X}}_{\Gamma(t)}^{\text{gf}}) &= 
    \frac{R}{2d} \tr \bigg [ \bigg ( \AA^T \AA (\AA^T \AA + \delta \II_d)^{-1} \big (\II_d-e^{-(\AA^T\AA + \delta \II_d)\Gamma(t)} \big )- \II_d \bigg )^2 \bigg ] \\
    &+
   \frac{ \widetilde{R}}{2d} \tr \bigg [ \bigg ( (\AA^T \AA + \delta \II_d)^{-1}\big [ \II_d - e^{-(\AA^T \AA + \delta \II_d)\Gamma(t)} \big ]\AA^T \bigg )^2 \bigg ] \\
    &+
    \frac{\widehat{R}}{2d} \tr \big ( e^{-2 (\AA^T \AA + \delta \II_d) \Gamma(t)} \big ).
\end{aligned}
\end{equation}
Under the learning rate assumptions in Theorem~\ref{thm:eventualrisk}, the limiting GF terms simplify
\begin{equation}
    \begin{aligned}
    \mathscr{L}(\bm{\mathscr{X}}_{\infty}^{\text{gf}})
    &=
    \frac{R}{2d} \tr \bigg [ (\AA^T\AA) \bigg ( \AA^T \AA (\AA^T \AA + \delta \II_d)^{-1} - \II_d \bigg )^2 \bigg ]\\
    & \qquad \qquad +
   \frac{ \widetilde{R}}{2d} \tr \bigg [ \bigg ( \AA(\AA^T \AA + \delta \II_d)^{-1}\AA^T - \II_n \bigg )^2 \bigg ] \\
    \mathcal{R}(\bm{\mathscr{X}}_{\infty}^{\text{gf}}) 
    &= 
    \frac{R}{2d} \tr \bigg [ \bigg ( \AA^T \AA (\AA^T \AA + \delta \II_d)^{-1} - \II_d \bigg )^2 \bigg ] +
   \frac{ \widetilde{R}}{2d} \tr \bigg [ \bigg ( (\AA^T \AA + \delta \II_d)^{-1} \AA^T \bigg )^2 \bigg ].
    \end{aligned}
\end{equation}

\paragraph{Excess risk in linear regression.} In the standard linear regression setup, we suppose that $\AA$ is generated by taking $n$ independent $d$-dimensional samples from a centered distribution $\mathcal{D}_f$ which we assume to be standardized (mean $0$ and expected sample-norm-squared $1$).  We let the matrix $\SSigma_f \in \mathbb{R}^{d \times d} $ be the feature covariance of $\mathcal{D}_f$, that is
\begin{equation} \label{eq:train_covariance}
    \SSigma_f \defas \Exp [ \aa \aa^T ],\quad \text{where}\quad \aa \sim \mathcal{D}_f.
\end{equation}
Suppose there is a linear (``ground truth'' or ``signal'') function $\beta : \mathbb{R}^d \to \mathbb{R}$, which for simplicity we suppose to have $\beta(0) = 0$.  In this case, we identify $\beta$ with a vector using the representation $\aa \mapsto \bbeta^T\aa$.  We suppose that our data is drawn from a distribution $\mathcal{D}$ on $\mathbb{R}^d \times \mathbb{R}$, with the property that
\[
\Exp[ \, b \, | \,  \aa \, ] = \bbeta^T \aa, \quad \text{where}\quad (\aa, b) \sim \mathcal{D},
\]
and the data $\aa \sim \mathcal{D}_f$.

Hence we suppose that $[\AA ~|~ \bb]$ is a $\R^{n \times d} \times \R^{n \times 1}$ matrix on independent samples from $\mathcal{D}$.  The vector $\xx_t$ represents an estimate of $\bbeta$, and the population risk is
\[
\mathcal{R}(\xx_t)
\defas \frac{1}{2} \Exp [ (b - \xx_t^T \aa)^2 | \xx_t] \quad \text{where}\quad (\aa, b) \sim \mathcal{D},
\]
where $(\aa,b)$ is an sample independent of $\xx_t$.  This can be evaluated in terms of the feature covariance matrix $\SSigma_f$ and the noise $\eta^2 \defas \Exp[ \, (b-\bbeta^T \aa)^2 \, ]$ to give
\begin{equation}\label{eq:lrq}
\mathcal{R}(\xx_t)
=\frac{1}{2}\eta^2 + \frac{1}{2} (\bbeta-\xx_t)^T \SSigma_f (\bbeta-\xx_t).
\end{equation}
It is important to note that the sequence $\{\xx_{\lfloor tn \rfloor} \}_{ t \ge 0}$ is generated from the iterates of SGD applied to the $\ell^2$-regularized least-squares problem \eqref{eq:rr}.

In the case that $(\aa, b)$ is jointly Gaussian, it follows that we may represent
\[
\aa = \SSigma_f^{1/2}\zz, \quad b = \bbeta^T \aa + \eta w,
\quad\text{where}\quad (\zz,w) \sim N(0, \II_d \oplus 1).
\]
Therefore, it follows that the iterates $\xx_{\lfloor nt \rfloor}$ are generated from the SGD algorithm applied to the problem:
\[
\min_{\xx} \frac{1}{2} \| \AA \xx - \bb\|^2_2 + \frac{\delta}{2} \|\xx\|^2_2
\quad\text{where}\quad
\bb = \AA \bbeta + \widehat{\eta} \ww ,
\]
and the vector $\ww$ is iid $N(0,1)$ random variables, independent of $\AA$.  This is also known as the generative model with noise.

Moreover, if $\mathcal{D}$ satisfies Assumption \ref{ass:sc} (with $\SSigma=\SSigma_f$) then the population risk $\mathcal{R}(\xx_{\lfloor tn \rfloor})$ is well approximated by $\Omega$:
\begin{theorem}\label{thm:AbR}
Suppose $(\AA, \bb)$ satisfy Assumption \ref{ass:sc}, $\delta >0$ and $\xx_0$ is iid centered subgaussian with $\Exp \|\xx_0\|^2 = \widehat{R}.$ For some $\epsilon > 0$, for all $T >0$, and for all $D>0$ there is a $C>0$ such that
\[
\Pr
\biggl(
\sup_{0 \leq t \leq T}
\left\| 
\begin{pmatrix}
\mathscr{L}(\sgd_{\lfloor tn \rfloor}) \\
\mathcal{R}(\sgd_{\lfloor tn \rfloor}) 
\end{pmatrix} - 
\begin{pmatrix}
\Psi_t  \\ 
\Omega_t
\end{pmatrix}
\right\|_2 > d^{-\epsilon}
\biggr)
\leq Cd^{-D},
\]
where $\Psi_t$ solves \eqref{eqa:V} and $\Omega_t$ solves \eqref{eqa:PLoss} with $\mathcal{R}$ given by \eqref{eq:lrq}.
\end{theorem}
\noindent We remark that under Assumption~\ref{ass:sc} (and in-distribution) that $\eta^2 =  \tfrac{\widetilde{R}}{d}$. In the case of out-of-distribition regression (see section below), we have that $\eta^2 \neq \frac{\widetilde{R}}{d}$ as the $\eta$ represents the population noise.  \\ 

\noindent The loss function $\mathcal{L}$ evaluated at GF is the same as in \eqref{eq:fullOmega_training} as is the limiting loss $\Omega_{\infty}$. For the test risk $\mathcal{R}$ in \eqref{eq:lrq} evaluated at GF, we have the following expressions for 
\begin{equation}
\begin{aligned} \label{eq:excess_risk_ERM_limit}
 \mathcal{R}(\bm{\mathscr{X}}_{\Gamma(t)}^{\text{gf}}) 
    &= \frac{R}{2d} \tr \bigg ( \SSigma_f \bigg ( \CC(t) \AA^T \AA-\II_d \bigg )^2 \bigg ) + \frac{\widetilde{R}}{2d} \tr \bigg ( \SSigma_f \AA^T\AA \CC^2(t) \bigg )\\
    &+ \frac{\widehat{R}}{2d} \tr \bigg ( \SSigma_f \exp \big (-2 (\AA^T \AA + \delta \II_d) \Gamma(t) \big ) \bigg ) + \frac{1}{2} \eta^2 \\
    \text{and} \quad \mathcal{R}(\bm{\mathscr{X}}_{\infty}^{\text{gf}})
    &= \frac{R}{2d} \tr \bigg ( \SSigma_f \big (\AA^T\AA (\AA^T\AA + \delta \II_d)^{-1} - \II_d \big )^2 \bigg )\\
    & + \frac{\widetilde{R}}{2d} \tr \bigg ( \SSigma_f \AA^T \AA \big ( \AA^T\AA + \delta \II_d \big )^{-2} \bigg )  + \frac{1}{2} \eta^2 \\
 \text{where} \quad \CC(t) 
    &\defas (\AA^T \AA + \delta \II_d)^{-1} \bigg (\II_d - \exp\big (-(\AA^T\AA + \delta \II_d)\Gamma(t) \big) \bigg ).
\end{aligned}
\end{equation}

\noindent Using Theorem \ref{thm:eventualrisk}, we conclude that in the case that $\gamma(t) \to 0$ as $t \to \infty$, the excess risk of SGD tends to $0$.  More interestingly, in the interpolation regime, $\mathscr{L}(\bm{\mathscr{X}}_\infty^\text{gf})=0$, i.e.\ the empirical risk tends to $0$.  In this case, even without taking $\gamma \to 0,$ the excess risk of SGD tends to $0$.  If on the other hand it does not tend to $0$ (i.e., $\gamma(t) \to \gamma$), we arrive at the formula for excess risk of SGD over the ridge estimator risk:
\begin{equation}\label{eq:ererm}
\begin{aligned}
\Omega_\infty- \mathcal{R}(\bm{\mathscr{X}}_\infty^{\text{gf}})
&=
\mathcal{L}(\bm{\mathscr{X}}_{\infty}^\text{gf}) \times \frac{\gamma}{2n} 
\frac{
\tr
\bigl(
(\nabla^2 \mathscr{L})
\SSigma_f
\bigl( \nabla^2 \mathscr{L} + \delta \II_d \bigr)^{-1}
\bigr)}
{
1
-
\frac{\gamma}{2n}
\tr
\bigl(
(\nabla^2 \mathscr{L})^2
\bigl( \nabla^2 \mathscr{L} + \delta \II_d \bigr)^{-1}
\bigr)
}\\
&=
\Psi_\infty
\times
\frac{\gamma}{n} 
\tr
\biggl(
\tfrac{
(\nabla^2 \mathscr{L})}
{\bigl( \nabla^2 \mathscr{L} + \delta \II_d \bigr)}
\SSigma_f
\biggr)
.
\end{aligned}
\end{equation}
\noindent We note that the right-hand-side is proportional to $\Psi_\infty$ (c.f.\ Theorem \ref{thm:eventualrisk}), and hence this excess risk due to SGD will be small if the limiting empirical risk $\Psi_\infty$ is small.  This also shows that the regularization term $\delta$ interacts with the excess risk due to SGD: if the spectrum of $\nabla^2 \mathcal{R}$ is heavy in that it has slowly decaying eigenvalues, the reduction in excess risk due to the regularization regularizer $\delta$ can be large.

\paragraph{(Out-of-distribution) linear regression.} As before, we suppose that the data matrix $\AA$ is generated by taking $n$ independent $d$-dimensional samples from a centered distribution $\mathcal{D}_f$ with feature covariance $\SSigma_f$ (see \eqref{eq:train_covariance}). We also suppose, as in the previous in-distribution example, that there is a linear (``ground truth" or ``signal") function $\beta : \mathbb{R}^d \to \mathbb{R}$ which we identify with the vector $\bbeta \in \mathbb{R}^d$ and for which $\mathbb{E}[b | \aa] = \bbeta^T \aa$ where $(\aa, b) \sim \mathcal{D}$ and the data $\aa \sim \mathcal{D}_f$. We will generate our target $b$ from the distribution $(\aa, b) \sim \mathcal{D}$. We then let $\xx_t$ be the iterates generated by SGD applied to the optimization problem
\[
    \min_{x \in \mathbb{R}^d} \, \frac{1}{2} \|\AA \xx-\bb\|^2 + \frac{\delta}{2} \|\xx\|^2,
\quad \text{where} \quad  (\aa_i, b_i) \sim \mathcal{D}.
\]

The main distinction from the previous example is that we measure our generalization error using a different distribution than $\mathcal{D}$. Explicitly, there exists another centered distribution $\widehat{\mathcal{D}_f}$ (standardized) with covariance features matrix $\widehat{\SSigma}_f \in \mathbb{R}^{d \times d}$ from which we generate a vector $\widehat{\aa} \sim \widehat{\mathcal{D}}_f$. Moreover, we generate a test point $(\widehat{\aa}, \widehat{b})$ from a new distribution $\widehat{\mathcal{D}}$ such that 
$\EE[ \widehat{b} | \widehat{\aa} ] = \bbeta^T \widehat{\aa}$ 
with the same $\bbeta$ as before and the distribution $\widehat{\mathcal{D}}$ has $\widehat{\aa}$-marginal $\widehat{\mathcal{D}}_f$. We measure the population risk, $\mathcal{R}: \mathbb{R}^d \to \mathbb{R}$ as 
\begin{equation}
\begin{aligned}
    \mathcal{R}(\xx_t) &\defas \frac{1}{2} \EE[ (\widehat{b}- \xx_t^T \widehat{\aa} )^2 | \xx_t] = \frac{1}{2} \eta^2 + \frac{1}{2} (\xx_t - \bbeta)^T \widehat{\SSigma}_f (\xx_t-\bbeta) \\
    \quad \text{where} \quad \eta^2 
    &\defas \EE[(\widehat{b}-\bbeta^T \widehat{\aa})^2 ].
\end{aligned}
\end{equation}
In this setting, we can again derive the limiting excess risk, which has a similar formula for $\mathcal{R}(\bm{\mathscr{X}}_{\Gamma(t)}^{\text{gf}})$ as in \eqref{eq:excess_risk_ERM_limit} by replacing $\SSigma_f$ with $\widehat{\SSigma}_f$.

\paragraph{Random features.} A central example where the quasi-random assumptions hold is the random features setting, which was introduced in \citep{Rahimi2008Random} for scaling kernel machines. Random features models provide a rich but tractable class of models to gain further insights into the generalization phenomena \citep{mei2019generalization,liao2020Random,adlam2020understanding,adlam2020neural,tripuraneni2021covariate}. These models are particularly of interest because of their connection to neural networks where the number of random features corresponds to model complexity \citep{jacot2018neural,neal1996priors,lee2018deep} and because of its use as a practical method for data analysis \citep{Rahimi2008Random,shankar2020neural}.

We suppose that the data matrix $\XX$ is generated by taking $n$ independent $n_0$-dimensional samples from a centered distribution $\mathcal{D}_f$ with feature covariance 
\[\SSigma_f \defas \EE[\XX_i^T \XX_i], \qquad \text{where $\XX_i \in \mathbb{R}^{1 \times n_0}$ and $\XX_i \sim \mathcal{D}_f$.}\]
We suppose for simplicity that $\XX$ is a data matrix having dimension $n \times n_0$ whose iid rows are drawn from a multivariate Gaussian with covariance $\SSigma_f$ and nice covariance structure:

\begin{assumption}\label{ass:RF_cov}
The distribution $\mathcal{D}_f$ is multivariate normal and
the covariance matrix $\SSigma_f$ of the random features data satisfies for some $C>0$
\[
\tfrac{1}{n_0}{\tr}{(\SSigma_f)}=1
\quad\text{and}\quad
\|\SSigma_f\|_{\text{op}} \leq C.
\]
\end{assumption}
\noindent This allows $\XX$ to be represented equivalently as $\XX = \ZZ\SSigma^{1/2}/\sqrt{n_0}$ for a iid standard Gaussian matrix $\ZZ$.
We suppose that $\WW$ is an $(n_0 \times d)$ iid feature matrix having standard Gaussian entries and independent of $\ZZ$ so that $\ZZ \SSigma^{1/2} \WW/\sqrt{n_0}$ is a matrix whose rows are standardized.  

We let $\sigma$ be an activation function satisfying:
\begin{assumption}\label{ass:RF_sigma}
The activation function satisfies for $C_0,C_1 \geq 0$ 
\[
|\sigma'(x)| \leq C_0e^{C_1|x|},
\quad\text{for all}\quad x \in \R,
\quad\text{and for standard normal $Z$,}\quad
\Exp \sigma(Z) = 0.
\]
\end{assumption}
\noindent  We note that from the outset, the growth rate of the derivative of the activation function implies a similar bound on the growth rate of the underlying activation function $\sigma$.
As before, we suppose the data $[\XX ~|~ \bb]$ is arranged in the matrix $\R^n \times (\R^{n_0} \times \R)$ where each row is an independent sample from $\mathcal{D}$. We now transform the data $\XX \in \mathbb{R}^{n \times n_0}$ by putting
\[
\AA =  \sigma(\XX\WW / \sqrt{n_0} ) \in \mathbb{R}^{n \times d},
\]
where $\WW \in \mathbb{R}^{n_0 \times d}$ is a matrix independent of $[\XX~|~\bb]$ of independent standard normals.\footnote{In \cite{mei2019generalization}, the distribution of the columns are taken as independent uniform vectors on the sphere $\sqrt{d}\,\mathbb{S}^{d-1}$.  The activation function $\sigma$ is a 1-Lipschitz function from $\R \to \R$ that is applied entrywise to the underlying matrix.}
The activation function $\sigma \, : \, \mathbb{R} \to \mathbb{R}$ is applied element-wise.
%

We introduce the following notation
\begin{equation}
\begin{aligned}
    \SSigma_{\sigma}(\WW) \defas &\EE[ \sigma(\XX_i \WW / \sqrt{n_0} )^T \sigma(\XX_i \WW / \sqrt{n_0}) \, | \, \WW] \\
    \text{and} \quad &\widehat{\sigma}(\WW) \defas \EE[\XX_i^T \sigma(\XX_i \WW / \sqrt{n_0}) | \WW].
\end{aligned}
\end{equation}
The population risk, $\mathcal{R} : \mathbb{R}^d \to \mathbb{R}$ as a random variable in $\XX$ and $\WW$, is
\begin{equation}\label{eq:RF_Risk}
\begin{aligned}
\mathcal{R}(\xx_t) &\defas
\Exp [ (b - \xx_t^T \sigma(\XX_i \WW / \sqrt{n_0}))^2 | \xx_t, \WW]\\
&= \eta^2 + \Exp[(\XX_i \bbeta -  \sigma(\XX_i \WW / \sqrt{n_0}) \xx_t )^2 \, | \, \xx_t, \WW]\\
&= \eta^2 + \bbeta^T \SSigma_f \bbeta + \xx_t^T \SSigma_{\sigma}(\WW) \xx_t - 2 \bbeta^T \widehat{\sigma}(\WW) \xx_t,\\
&\qquad \qquad \qquad \qquad \text{where $(\XX_i, b) \sim \mathcal{D}$ and $\Exp[\, b \, | \, \XX_i] = \XX_i \bbeta$.}
\end{aligned}
\end{equation}
The $\ell^2$-regularized least-squares problem is now
\[
\min_{\xx} \frac{1}{2} \| \AA \xx - \bb\|^2_2 + \frac{\delta}{2} \|\xx\|^2_2
\quad\text{where}\quad
\bb = \XX \bbeta + \eta \ww,
\]
which is the random features regression.  This should be compared to a two-layer neural network model, in which the hidden layer has dimension $n_0$.  However, the hidden layer weights are simply generated randomly in advance and are left untrained.  The optimization is only performed on the final layers' weights ($\xx$). 

\begin{theorem} \label{thm:random_features}
Suppose that $n,d,n_0$ are proportionally related.
Suppose that the data matrix $\XX$ satisfies Assumption \ref{ass:RF_cov},
and the random features $\WW$ are iid standard normal.
Suppose $\bb = \XX \bbeta + \eta \ww$ with $\bbeta,\ww$ independent isotropic subgaussian vectors with $\Exp \|\bbeta\|^2_2 = 1/n_0$ and $\Exp \|\ww\|^2_2 = 1$ and $\eta$ bounded independent of $n$.
Suppose the activation function satisfies Assumption \ref{ass:RF_sigma}.
Suppose the initialization $\xx_0$ is iid centered subgaussian with $\Exp \|\xx_0\|^2_2 = \widehat{R}.$
Then for some $\epsilon > 0$, for all $T >0$, and for all $D>0$ there is a $C>0$ such that
\[
\Pr
\biggl(
\sup_{0 \leq t \leq T}
\left\| 
\begin{pmatrix}
\mathscr{L}(\sgd_{\lfloor tn \rfloor}) \\
\mathcal{R}(\sgd_{\lfloor tn \rfloor}) 
\end{pmatrix} - 
\begin{pmatrix}
\Psi_t  \\ 
\Omega_t
\end{pmatrix}
\right\| > d^{-\epsilon}
\biggr)
\leq Cd^{-D},
\]
where $\Psi_t$ solves \eqref{eqa:V} and $\Omega_t$ solves \eqref{eqa:PLoss} with $\mathcal{R}$ given by \eqref{eq:RF_Risk}.
\end{theorem}

Finally, as in \eqref{eq:ererm} we derive the excess risk of SGD ($\gamma(t) \to \gamma$) over ridge regression:
\begin{equation} \begin{aligned}
\label{eq:RF_er}
\Omega_\infty- \mathcal{R}(\bm{\mathscr{X}}_\infty^{\text{gf}})
&= \mathcal{L}(\bm{\mathscr{X}}_{\infty}^{\text{gf}}) \times
\frac{\gamma}{2n} 
\frac{
\tr
\bigl(
(\nabla^2 \mathscr{L})
\SSigma_{\sigma}(\WW)
\bigl( \nabla^2 \mathscr{L} + \delta \II_d \bigr)^{-1}
\bigr)}
{
1
-
\tfrac{\gamma}{2n}
\tr
\bigl(
(\nabla^2 \mathscr{L})^2
\bigl( \nabla^2 \mathscr{L} + \delta \II_d \bigr)^{-1}
\bigr)
}\\
&=
\Psi_\infty
\times
\frac{\gamma}{2n} 
\tr
\biggl(
\frac{
(\nabla^2 \mathscr{L})}
{\bigl( \nabla^2 \mathscr{L} + \delta \II_d \bigr)}
(\SSigma_{\sigma}(\WW))
\biggr)
.
\end{aligned}
\end{equation}

\subsubsection{Discussion of Theorem~\ref{thm:lppp} and motivating examples}
In this section, we discuss the equivalence of SGD and homogenized SGD under quadratic statistics $\mathcal{R}$ satisfying Assumption~\ref{ass:risk} and quasi-random assumptions on the data matrix $\AA$, initialization $\xx_0$, and target vector $\bb$ (see Appendix~\ref{sec:quasi_random}).
As the proof of Theorem~\ref{thm:lppp} is quite mathematically involved and it does not add to the interpretation of the risk trajectories, we relegate this proof to \citep[Theorem 1.3]{paquette2022homogenization}. 

The proofs of Theorems~\ref{thm:Ab}, \ref{thm:AbR}, and \ref{thm:random_features} follow immediately from Theorem~\ref{thm:lppp} and Theorem~\ref{thm:trainrisk}. In each of the cases, the extra assumptions on the initialization, signal, and data matrix simplify the GF terms in \eqref{eqa:VLoss} and \eqref{eqa:PLoss}. 

\section{The Volterra SLD class and Concentration of HSGD} \label{sec:sld_concentration}

First, we state the Volterra equation for streaming SGD that we referenced in the main text.
\begin{mdframed}[style=exampledefault]
{\textbf{Volterra Dynamics (streaming).} The following deterministic dynamical system is the high-dimensional equivalent for $\mathscr{L}(\YY_t)$ and $\mathcal{R}(\YY_t)$, respectively},
\begin{align} 
&\Psi^{\text{s}}_t
=
\mathscr{S}\bigl( 
\sgf_{\gamma t}\bigr)
+
\int_0^t 
K^{\text{s}}(t,u; \nabla^2 \mathscr{S}) 
\Psi^{\text{s}}_u
\dif u,
\quad\text{for}\quad t \geq 0
 & \text{(Empirical risk)} \label{eqb:VLoss} \\ 
&\Omega^{\text{s}}_t
=
\mathcal{R}\bigl( 
\sgf_{\gamma t}\bigr)
+\int_0^t 
K^{\text{s}}(t,u; \nabla^2 \mathcal{R}) 
\Psi^{\text{s}}_u
\dif u
\quad\text{for}\quad t \geq 0& \text{(Population risk)} \label{eqb:PLoss}
\end{align}
where the kernel $K$, for any $d \times d$ matrix $\PP$, is 
{\small
\begin{equation} \begin{aligned}\label{eqb:V}
K^{\text{s}}(t,u ; \PP) = 
\gamma^2
\tr
\biggl(
(\nabla^2 \mathscr{S})
\PP
\exp\bigl( -2\gamma(\nabla^2 \mathscr{S})( t - u)\bigr)
\biggr)
\end{aligned}
\end{equation} }
and \textit{GF for streaming}, $\sgf_{t}$, is the solution of
\[
\dif \sgf_{t} \defas -\nabla {\mathscr{S}}(\sgf_{t}) \dif t
\quad\text{and}\quad \sgf_{0} = \YY_0.
\]
\end{mdframed}
This result for streaming and the similar multi-pass SGD dynamics satisfy a large class of expressions. We will enlarge the class of SLDs that we consider to what we will call as the \textit{Volterra SLD class} defined as 
\begin{equation}\label{eq:VSLD}
\dif \XX_t \defas
-\gamma(t) \nabla {f}(\XX_t) \dif t
+ \gamma(t)\sqrt{\mathscr{L}(\XX_t) \mathscr{M}_t + \mathscr{A}_t}\dif \BB_t,
\end{equation}
where $\mathscr{M}_t$ and $\mathscr{A}_t$ are two deterministic positive definite functions which we assume to be normalized to satisfy:
\begin{assumption}\label{you_ass}
The covariance processes $\mathscr{M}$ and $\mathscr{A}$ satisfy for some absolute constants $c>0$ and $\epsilon > 0$
\[
\sup_{t \geq 0} \bigl(\tr \mathscr{M}_t + \tr \mathscr{A}_t \bigr) \leq c < \infty
\quad\text{and}
\quad
\sup_{t \geq 0} \bigl(\|\mathscr{M}_t\|_{op} + \|\mathscr{A}_t\|_{op} \bigr) < d^{-\epsilon}.
\]
\end{assumption}
The $\mathcal{M}_t$ represents noise in the data either because the data is randomly sampled or the data is transformed by multiplicative transformation. In contrast $\mathcal{A}_t$ represents an additive noise at each step.  It is any noise which does not multiple the state $\xx$, for example, label noise.

\paragraph{Volterra SLD class.} The Volterra SLD class is so-named because the expected loss satisfies a Volterra type integral equation.  Define for $t \geq s \geq 0$ and positive semidefinite $\PP$,
{\small \begin{equation}\label{eq:V}
\Gamma(t) = \int_0^t \gamma(s)\,\dif s,
\,
\text{and}
\,
\left\{
\begin{aligned}
K(t,s ; \PP) &= 
\gamma^2(s)
\tr
\biggl(
\mathscr{M}_{s}
\PP
\exp\bigl( -2(\AA^T \AA + \delta\II_d)( \Gamma(t) - \Gamma(s))\bigr)
\biggr) \\
A(t,s; \PP) &= 
\gamma^2(s)
\tr
\biggl(
\mathscr{A}_{s}
\PP
\exp\bigl( -2(\AA^T \AA + \delta\II_d)( \Gamma(t) - \Gamma(s) )\bigr)
\biggr)
\end{aligned}\right\}.
\end{equation}
}
We shall suppose throughout that $\gf_{t}$ is the canonical GF
\[
\dif \gf_{t} = -\nabla f ( \gf_{t})
\quad\text{and}\quad \gf_0 = \XX_0.
\]
The loss $\mathscr{L}(\XX_t)$ concentrates around the solution $\Psi_t$ of the Volterra integral equation (see Theorem \ref{thm:conc} for a precise formulation):
\begin{equation}\label{eq:VLoss}
\Psi_t
=
\mathscr{L}\bigl( 
\gf_{\Gamma(t)}\bigr)
+
\int_0^t 
A(t,s; \AA^T\AA)
\dif s +
\int_0^t 
K(t,s; \AA^T\AA) 
\Psi_s
\dif s.
\end{equation}
We give a formal proof of the concentration result in Section \ref{sec:VolterraConcentration}.  For other quadratics $\mathcal{R}$, the loss $\mathcal{R}(\XX_t)$ concentrates around
\begin{equation}\label{eq:PLoss}
\Omega_t
=
\mathcal{R}\bigl( 
\gf_{\Gamma(t)}\bigr)
+
\int_0^t 
A(t,s; \nabla^2 \mathcal{R})
\dif s +
\int_0^t 
K(t,s; \nabla^2 \mathcal{R}) 
\Psi_s
\dif s.
\end{equation}

\subsection{Volterra Concentration}\label{sec:VolterraConcentration}
In this section, we prove the concentration result. In this section, we prove that homogenized SGD concentrates around its mean provided that the expected loss in is in the Volterra SLD class. The result, in this section, Theorem~\ref{thm:conc}, is more general than Theorem~\ref{thm:trainrisk} which follows by setting $\mathcal{M}_t \equiv \tfrac{1}{n} \nabla^2 \mathscr{L}$ and $\mathcal{A}_t \equiv 0$. 
\begin{theorem}\label{thm:conc}
Under Assumption \ref{you_ass}, the loss $\mathscr{L}(\XX_t)$ concentrates around $\Psi_t$ the solution of the Volterra equation
\[
\Psi_t
=
\mathscr{L}\bigl( 
\gf_{\Gamma(t)}\bigr)
+
\int_0^t 
A(t,s; \AA^T\AA)
\dif s
+
\int_0^t 
K(t,s; \AA^T\AA) 
\Psi_s
\dif s,
\]
in that for any $T,D > 0$ there is a $C(T,D,\|\AA\|_{op}, \|\bb\|,\epsilon) > 0$ sufficiently large that
\[
\Pr
\bigl[
\sup_{0 \leq t \leq T} | \mathscr{L}(\XX_t) - \Psi_t | > C d^{-\epsilon/2}
\bigr]
\leq Cd^{-D}.
\]
Furthermore, for another quadratic $\mathcal{R}(\xx) = 1/2 \xx^T \TT \xx + \uu^T \xx + c$ with $\TT$ symmetric matrix having $\|\nabla^2 \mathcal{R}\|_{\text{op}} \leq C$, $\| \nabla \mathcal{R}\|_2 \le C$, $\|\nabla \mathcal{R}(0)\| \le 1$ are independent of the Brownian motion,
\[
\Pr
\biggl[
\sup_{0 \leq t \leq T} \biggl|
-\mathcal{R}(\XX_t)
+
\mathcal{R}\bigl( 
\gf_{\Gamma(t)}\bigr)
+
\int_0^t 
A(t,s; \nabla^2 \mathcal{R})
\dif s
+
\int_0^t 
K(t,s;  \nabla^2 \mathcal{R}) 
\Psi_s
\dif s
\biggr| > C d^{-\epsilon/2}
\biggr]
\leq Cd^{-D}.
\]
\end{theorem}
\begin{proof}
\noindent \emph{Step 1. Volterra equation for the expected loss}. 

Define $\QQ_t = \exp( (\AA^T \AA + \delta \II_d) \Gamma(t))$ 
and apply It\^o's rule to $\QQ_t \XX_t$, derive
\[
\dif(\QQ_t \XX_t) = \gamma(t) \QQ_t \AA^T \bb \dif t
+ \gamma(t)\QQ_t \sqrt{\mathscr{L}(\XX_t) \mathscr{M}_t + \mathscr{A}_t}\dif \BB_t.
\]
Hence
\[
\QQ_t\XX_t= \QQ_0 \XX_0 + \int_0^t 
\gamma(s) \QQ_s \AA^T \bb \dif s
+ \int_0^t \gamma(s)\QQ_s \sqrt{\mathscr{L}(\XX_s) \mathscr{M}_s + \mathscr{A}_s}\dif \BB_s.
\]
Note that on setting $\mathscr{M}_s=\mathscr{A}_s=0$, this gives GF $\gf_{\Gamma(t)}$, and hence we have representation
\[
\XX_t= \gf_{\Gamma(t)}
+ \QQ_t^{-1} \int_0^t \gamma(s)\QQ_s \sqrt{\mathscr{L}(\XX_s) \mathscr{M}_s + \mathscr{A}_s}\dif \BB_s.
\]
Expanding the quadratic,
\begin{equation}\label{eq:rawL}
\begin{aligned}
\mathscr{L}( \XX_t)
&=
\mathscr{L}\bigl( 
\gf_{\Gamma(t)}\bigr)
+
\nabla \mathscr{L}(\gf_{\Gamma(t)})^T\QQ_t^{-1} \int_0^t \gamma(s)\QQ_s \sqrt{\mathscr{L}(\XX_s) \mathscr{M}_s + \mathscr{A}_s}\dif \BB_s
\\
&+
\frac{1}{2} \biggl\|\AA\QQ_t^{-1} \int_0^t \gamma(s)\QQ_s \sqrt{\mathscr{L}(\XX_s) \mathscr{M}_s + \mathscr{A}_s}\dif \BB_s\biggr\|^2.
\end{aligned}
\end{equation}
It follows that with $\Filt_t$ the sigma-algebra generated by $(\XX_0,\mathscr{L}, (\BB_s: 0 \leq s \leq t))$
if we compute the $\Filt_0$-conditional expectation, the Brownian integral vanishes, and we are left with two contributions from the second norm-squared process
\begin{equation}\label{eq:rawV}
\begin{aligned}
\Exp \bigl[ \mathscr{L}( \XX_t) ~\vert~ \Filt_0]
&=
\mathscr{L}\bigl( 
\gf_{\Gamma(t)}\bigr)
+
\frac{1}{2} \int_0^t 
\gamma^2(s)
\tr
\biggl(
\AA^T\AA
\QQ_t^{-2}\QQ_s^2\mathscr{A}_s
\biggr)
\dif s
\\
&+
\frac{1}{2} \int_0^t 
\gamma^2(s)
\tr
\biggl(
\AA^T\AA
\QQ_t^{-2}\QQ_s^2\mathscr{M}_s
\biggr)
\Exp \bigl[ \mathscr{L}( \XX_s) ~\vert~ \Filt_0]
\dif s.
\end{aligned}
\end{equation}
This is the claimed Volterra equation (see e.g., $\Psi_t \defas \EE[\mathscr{L}(\XX_t) ~|~ \Filt_0]$
in \eqref{eqa:VLoss}; here $\mathscr{A}_t = 0$.) 

\noindent \emph{Step 2. High probability boundedness of $\mathscr{L}$}.
We observe before beginning that many of the quantities that appear in the expressions above are bounded.
The GF $\gf_{\Gamma(t)}$ satisfies a uniform bound, solely in terms of its initial conditions and, in particular, the boundedness of $\mathscr{L}(\gf_{\Gamma(t)})$ satisfies $\mathscr{L}(\gf_{\Gamma(t)}) \le \mathscr{L}(\XX_0)$.  The matrix $\QQ_t^{-1}\QQ_s$ is uniformly bounded in norm by $1$ for all $t \geq s$.  We have also assumed that $\|\AA\|_{\text{op}}$ and $\|\bb\|_2$ are bounded.
By applying It\^o's formula to the norm $u_t = \frac12 \|\XX_t\|^2_2$, we have from \eqref{eq:VSLD} that
\[
\dif u_t
= -\gamma(t) \XX_t^T (\AA^T (\AA \XX_t - \bb)) \dif t
+ \gamma(t)\XX_t^T \sqrt{\mathscr{L}(\XX_t) \mathscr{M}_t + \mathscr{A}_t}\dif \BB_t
+ \frac{\gamma^2(t)}{2}\tr\bigl(\mathscr{L}(\XX_t) \mathscr{M}_t + \mathscr{A}_t\bigr)\dif t
\]

From the norm boundedness of $\AA$ and $\bb$, we can bound $\mathscr{L}(\XX_t) \leq 2\|\AA\|^2_{\text{op}} u_t + 2\|\bb\|^2_2 \leq C(u_t + 1)$.  Likewise, increasing $C$ as need be, using the boundedness of $\gamma(t)$, $\tr\mathscr{M}_t$ and $\tr \mathscr{A}_t$, we conclude
\[
d \langle u_t \rangle 
= \gamma^2(t)\tr\bigl( \XX_t \XX_t^T (\mathscr{L}(\XX_t) \mathscr{M}_t + \mathscr{A}_t)\bigr)
\leq C(u_t+1)^2.
\]
It follows that $z_t \defas \log( 1 + u_t ) - Ct$ is supermartingale with $\langle z_t \rangle \leq C$ for some sufficiently large $C$ and all $t \leq T$.  Hence with probability at least $1-e^{-2C(T)(\log d)^{3/2}}$,
\[
z_t \leq (\log d)^{3/4}
\]
for all $t \leq T$.
On this same event it follows for a sufficiently large cosntant $C>0$
\[
f(t) = \mathscr{L}(\XX_t) + \delta u_t 
\quad\text{and}
\quad
\mathscr{L}(\XX_t) \leq C(u_t + 1) \leq C^2e^{Ct+(\log d)^{3/4}}
\]
for all $t \leq T$.


\noindent \emph{Step 3. Concentration of the loss}.  We may now control the difference of the loss from its expectation.  Specifically, in comparing 
\eqref{eq:rawL} and \eqref{eq:rawV}, we may express the difference $\Delta_t \defas \mathscr{L}(\XX_t)-\Exp[ \mathscr{L}(\XX_t) ~|~ \Filt_0]$ as
\begin{equation}\label{eq:Deltat}
\begin{aligned}
\Delta_t &= \MM_t^{(1)} + \MM_t^{(2)} + \frac{1}{2} \int_0^t
\gamma^2(s)
\tr
\biggl(
\AA^T\AA
\QQ_t^{-2}\QQ_s^2\mathscr{M}_s
\biggr)
\Delta_s
\dif s, \quad\text{where} \\
\MM_t^{(1)}
&=\nabla \mathscr{L}(\gf_{\Gamma(t)})^T\QQ_t^{-1} \int_0^t 
\gamma(s)\QQ_s \sqrt{\mathscr{L}(\XX_s) \mathscr{M}_s + \mathscr{A}_s}\dif \BB_s, \quad\text{and} \\
\MM_t^{(2)}
&=
\frac{1}{2} \biggl\|\AA\QQ_t^{-1} \int_0^t 
\gamma(s)\QQ_s \sqrt{\mathscr{L}(\XX_s) \mathscr{M}_s + \mathscr{A}_s}\dif \BB_s\biggr\|^2\\
& \qquad \quad -
\frac{1}{2} \int_0^t 
\gamma^2(s)
\tr
\biggl(
\AA^T\AA
\QQ_t^{-2}\QQ_s^2(\mathscr{L}(\XX_s)\mathscr{M}_s+\mathscr{A}_s)
\biggr)
\dif s.
\end{aligned}
\end{equation}

We claim that both processes $\MM_t^{(1)}$ and $\MM_t^{(2)}$ are small, whose proof we defer.  Specifically, with probability $1-C(T)e^{-(\log d)^{3/2}}$ we have
\[
\max_{0 \leq t \leq T}\bigl\{ |\MM_t^{(1)}| + |\MM_t^{(2)}| \bigr\} \leq d^{-3\epsilon/4}.
\]
From the uniform boundedness in norm of $\AA^T\AA$, we then conclude from \eqref{eq:Deltat} for all $t \leq T$
\[
|\Delta_t|
\leq 
d^{-3\epsilon/4}
+\int_0^t \|\AA^T\AA\|_{\text{op}} |\Delta_s|\dif s.
\]
Using Gronwall's inequality, 
\[
|\Delta_t| \leq \|\AA^T\AA\|_{\text{op}}^{-1} \bigl( e^{\|\AA^T\AA\|_{\text{op}} t}-1\bigr) d^{-3\epsilon/4}.
\]
Thus we conclude by increasing the constants in the claimed bound that the desired inequality holds.

\noindent \emph{Step 4 (Deferred). Concentration of the martingales}.
We introduce two martingales, for each fixed $t \in [0,T]$,
\[
\begin{aligned}
\MM_u^{(1,t)}
&=
\nabla \mathscr{L}(\gf_{\Gamma(t)})^T\QQ_t^{-1} \int_0^u 
\gamma(s)\QQ_s \sqrt{\mathscr{L}(\XX_s) \mathscr{M}_s + \mathscr{A}_s}\dif \BB_s. \\
\MM_u^{(2,t)}
&=
\frac{1}{2} \biggl\|\AA\QQ_t^{-1} \int_0^u 
\gamma(s)\QQ_s \sqrt{\mathscr{L}(\XX_s) \mathscr{M}_s + \mathscr{A}_s}\dif \BB_s\biggr\|^2\\
& \qquad \quad -
\frac{1}{2} \int_0^u 
\gamma^2(s)
\tr
\biggl(
\AA^T\AA
\QQ_t^{-2}\QQ_s^2(\mathscr{L}(\XX_s)\mathscr{M}_s+\mathscr{A}_s)
\biggr)
\dif s.
\end{aligned}
\]
We first show that if we fix any $t \leq T$, then for all $d$ sufficiently large with respect to $T$ and with probability at least $1-2e^{-(\log d)^{3/2}}$,
\[
\max_{0 \leq u \leq t}\bigl\{ |\MM_u^{(1,t)}| + |\MM_u^{(2,t)}| \bigr\} \leq d^{-7\epsilon/8}.
\]
We will then need to use a meshing argument to complete the argument.  We show the details for the first.  Those for the second are similar.

We simply need to bound the quadratic variation of each.  Note
\[
\langle
\MM_u^{(1,t)}
\rangle
=
\int_0^u 
\gamma^2(s)
\tr\bigl( \QQ_t^{-1} \QQ_s \nabla \mathscr{L}(\gf_{\Gamma(t)}) \nabla \mathscr{L}(\gf_{\Gamma(t)})^T \QQ_t^{-1} \QQ_s (\mathscr{L}(\XX_s) \mathscr{M}_s + \mathscr{A}_s) \bigr)\dif u.
\]
Here we use the norm boundedness of $\mathscr{M}_t + \mathscr{A}_t$, by $2d^{-\epsilon}$.  We further bound the other terms in norm to produce
\[
\langle
\MM_u^{(1,t)}
\rangle
\leq
2d^{-\epsilon} \|\AA^T\AA\|_{\text{op}}^2 \bigl(C^2 e^{Cu + (\log d)^{3/4}}\bigr) \mathscr{L}(\gf_{\Gamma(t)}). 
\]
 We note that $\mathscr{L}(\gf_{\Gamma(t)}) \le \mathscr{L}(\gf_0)$. Hence with probability at least $1 - e^{-(\log d)^{3/2}}$ (for all $d$ sufficiently large with respect to $T,\|\AA\|_{\text{op}},\|\bb\|_2, \epsilon$),
\[
\max_{0 \leq u \leq t}
|
\MM_u^{(1,t)}
|
\leq
d^{-7\epsilon/8}/2.
\]

\noindent \emph{Step 5 (Deferred). Mesh argument}. Finally, we use a union bound to gain the control from Step 4 over a mesh of $[0,T]$ of spacing $d^{-100}$.  From the union bound, we therefore have for all these mesh points $\{t_k\}$
\[
\max_k \max_{0 \leq u \leq t_k}\bigl\{ |\MM_u^{(1,t_k)}| + |\MM_u^{(2,t_k)}| \bigr\} \leq d^{-7\epsilon/8},
\]
and this holds with probability $1-2Td^{100}e^{-(\log d)^{3/2}}$.  For $t \in [t_k,t_{k+1}]$, we just use that
\[
\|
\nabla \mathscr{L}(\gf_{\Gamma(t)})^T \QQ_t^{-1}
-
\nabla \mathscr{L}(\gf_{\Gamma(t_{k+1})})^T \QQ_{t_{k+1}}^{-1}\|
\leq C(T,\AA, \bb)d^{-100},
\]
and thus on the event that $\mathscr{L}(\XX_s)$ is bounded, we have for $t \in [t_k,t_{k+1}]$
\[
|\MM_t^{(1)} - \MM_t^{(1,t_{k+1})}| \leq C(T,\AA, \bb)d^{-50}
\]
for all $d$ sufficiently large with respect to $T$, $\|\AA\|_{\text{op}},$ and $\|\bb\|_2$.

\noindent \emph{Step 6. Other quadratics}. Hence, if we take $\Psi_t$ as a solution to the Volterra equation
\[
\Psi_t
=
\mathscr{L}\bigl( 
\gf_{\Gamma(t)}\bigr)
+
\frac{1}{2} \int_0^t 
\gamma^2(s)
\tr
\biggl(
\AA^T\AA
\QQ_t^{-2}\QQ_s^2\mathscr{A}_s
\biggr)
\dif s
+ \frac{1}{2}
\int_0^t 
\gamma^2(s)
\tr
\biggl(
\AA^T\AA
\QQ_t^{-2}\QQ_s^2\mathscr{M}_s
\biggr)
\Psi_s
\dif s,
\]
then we have a high-quality approximation for the loss $\mathscr{L}(\XX_t)$, and moreover, applying It\^o's equation, we may always represent another quadratic $\mathcal{R} :\R^d \to \R$ of the SLD by (analogously to \eqref{eq:rawL})
\[
\begin{aligned}
&\mathcal{R}(\XX_t)
=
\mathcal{R}\bigl( 
\gf_{\Gamma(t)}\bigr)
+
\nabla \mathcal{R}(\gf_{\Gamma(t)})^T \QQ_t^{-1} \int_0^t 
\gamma(s)
\QQ_s \sqrt{\mathscr{L}(\XX_s) \mathscr{M}_s + \mathscr{A}_s}\dif \BB_s
\\
&+ \frac{1}{2} \bigg (\QQ_t^{-1} \int_0^t \gamma(s)
\QQ_s \sqrt{\mathscr{L}(\XX_s) \mathscr{M}_s + \mathscr{A}_s}\dif \BB_s \bigg )^T (\nabla^2 \mathcal{R}) \QQ_t^{-1} \int_0^t \gamma(s)
\QQ_s \sqrt{\mathscr{L}(\XX_s) \mathscr{M}_s + \mathscr{A}_s}\dif \BB_s
\end{aligned}
\]
By comparing this to the same expression, where we replace the losses $\mathscr{L}(\XX_s)$ by $\Psi_s$ and compute expectations over the Brownian terms, we arrive at (compare \eqref{eq:rawV})
\[
\begin{aligned}
\mathcal{R}( \XX_t)
&=
\MM_t^{(3)}
+
\mathcal{R}\bigl(\gf_{\Gamma(t)}\bigr)
+
\frac{1}{2} \int_0^t 
\gamma^2(s)
\tr
\biggl(
(\nabla^2 \mathcal{R})
\QQ_t^{-2}\QQ_s^2\mathscr{A}_s
\biggr)
\dif s
\\
&+ \frac{1}{2}
\int_0^t 
\gamma^2(s)
\tr
\biggl(
(\nabla^2 \mathcal{R})
\QQ_t^{-2}\QQ_s^2\mathscr{M}_s
\biggr)
\Psi_s
\dif s.
\end{aligned}
\]
Provided the Hessian $(\nabla^2 \mathcal{R})$ and gradient $\nabla \mathcal{R}(0)$ are bounded independently of $d$ uniformly on $T$, the concentration of $\MM_t^{(3)}$ now follows exactly as in Steps 4 and 5.
\end{proof}

\section{Limiting values of the excess risk} \label{sec:limit_excess_risk}
In this section, we prove the limiting excess risk values, Theorem~\ref{thm:eventualrisk}. We will, in fact, prove a more general version of Theorem~\ref{thm:eventualrisk} which holds for a wider class, so called the Volterra SLD class, as discussed in \eqref{eq:V}. Theorems~\ref{thm:sgdlimit} (constant learning rate) and \ref{thm:eventualrisk_SLD} (time dependent learning rate) immediately imply Theorem~\ref{thm:eventualrisk} by setting $\mathscr{M}_t \equiv \frac{1}{n} \nabla^2 \mathscr{L}$ and $\mathcal{A}_t \equiv 0$. By using the Volterra SLD class, we also recover the result for streaming, Theorem~\ref{thm:bootstraprisk}.

\paragraph{Excess risk in the constant case.} Under the stronger assumptions of constant learning rate, and constant variance profile, this can be further simplified.   That is, suppose
\begin{assumption}\label{flat_ass}
Suppose that the covariance processes $\mathscr{M}$ and $\mathscr{A}$ 
are constant and satisfy for some absolute constants $c>0$ and $\epsilon > 0$
\[
(\gamma(t), \mathscr{M}_t, \mathscr{A}_t)
\equiv
(\gamma, \mathscr{M}, \mathscr{A}),
\quad\text{where}\quad
\bigl(\tr \mathscr{M} + \tr \mathscr{A} \bigr) \leq c < \infty
\quad\text{and}
\quad
\bigl(\|\mathscr{M}\|_{op} + \|\mathscr{A}\|_{op} \bigr) < d^{-\epsilon}.
\]
\end{assumption}
\noindent Under this assumption the kernels in the Volterra equation simplify to be:
\begin{equation}\label{eq:VC}
\begin{aligned}
K(t,s ; \PP) 
=
K(t-s ; \PP) 
&= 
\gamma^2
\tr
\biggl(
\mathscr{M}
\PP
\exp\bigl( -2\gamma(\AA^T \AA + \delta \II_d)(t-s)\bigr)
\biggr),\\
A(t,s; \PP) 
=
A(t-s ; \PP) 
&= 
\gamma^2
\tr
\biggl(
\mathscr{A}
\PP
\exp\bigl( -2\gamma (\AA^T \AA + \delta \II_d)( t-s )\bigr)
\biggr)
\end{aligned}
\end{equation}
The theory of convolution-type Volterra equations is substantially simpler than those of non-convolution type.  In particular, we can completely recover the rates of convergence and the limiting loss, as well as convergence guarantees (note that if the training loss of the underlying GF does not tend to $0$ and or $\mathscr{A} \neq 0$, then the loss does not tend to $0$, and so this is neighborhood convergence).
\begin{theorem}[Limit risk values, constant learning rate ]\label{thm:sgdlimit} Suppose the learning rate is constant, $\gamma(t) \equiv \gamma$. 
Under Assumption \ref{flat_ass}, the Volterra SLD is (neighborhood) convergent if and only if
\begin{equation}\label{eq:knorm}
\mathcal{I}(\gamma) \defas 
\int_0^\infty K(t; \AA^T \AA)\,\dif t
=
\frac{\gamma}{2}
\tr
\biggl(
\mathscr{M}(\AA^T \AA)
\bigl( \AA^T \AA + \delta \II_d)\bigr)^{-1}
\biggr) < 1.
\end{equation}
In the case that $\mathcal{I}(\gamma) < 1,$ $\Psi_t$ converges as $t\to\infty$ to
\begin{equation}\label{eq:Psi}
\Psi_\infty \defas
(1-\mathcal{I})^{-1}
\biggl(
\mathscr{L}(\gf_{\infty})
+
\frac{\gamma}{2}
\tr
\bigl(
\mathscr{A}(\AA^T \AA)
\bigl( \AA^T \AA + \delta \II_d)\bigr)^{-1}
\bigr)
\biggr).
\end{equation}
Likewise, the population risk $\Omega_t$ converges as $t\to\infty$ to
\begin{equation}\label{eq:Omega}
\Omega_\infty \defas
\mathcal{R}(\gf_{\infty})
+
\frac{\gamma}{2}
\tr
\bigl(
(\mathscr{A}+\mathscr{M}\Psi_\infty)(\nabla^2 \mathcal{R})
\bigl( \AA^T \AA + \delta \II_d )\bigr)^{-1}
\bigr).
\end{equation}
\end{theorem}

\begin{proof}[Proof of Theorem~\ref{thm:sgdlimit}] follows immediately from limiting values of renewal equations (see \citep{gripenberg1980volterra} and \citep{Asmussen}). 
\end{proof}

By setting $\mathcal{M} = \frac{1}{n} \nabla^2 \mathscr{L}$ and $\mathcal{A} \equiv 0$, we recover the multi-pass SGD setting discussed in the main portion of this paper. When the learning rate satisfies $\gamma(t) \equiv \gamma \in (0, 2(\tfrac{1}{n} \tr \big \{ \tfrac{(\AA^T\AA)^2}{\AA^T \AA + \delta \II_d} \big \})^{-1}$, it follows that $\mathcal{I}(\gamma) < 1$. Consequently, Theorem~\ref{thm:sgdlimit} proves Theorem~\ref{thm:eventualrisk} when the learning rate is constant. 

\paragraph{Excess risk when learning rate is time dependent.} In the case that the learning rate is time dependent, we prove the following result for the limiting dynamics under the expanded Volterra SLD class. Here we will still assume that the covariance processes $\mathcal{M}$ and $\mathcal{A}$ are constant.

\begin{assumption}\label{flat_ass_SLD}
Suppose that the covariance processes $\mathscr{M}$ and $\mathscr{A}$ 
are constant and satisfy for some absolute constants $c>0$ and $\epsilon > 0$
\[
(\mathscr{M}_t, \mathscr{A}_t)
\equiv
(\mathscr{M}, \mathscr{A}),
\quad\text{where}\quad
\bigl(\tr \mathscr{M} + \tr \mathscr{A} \bigr) \leq c < \infty
\quad\text{and}
\quad
\bigl(\|\mathscr{M}\|_{op} + \|\mathscr{A}\|_{op} \bigr) < d^{-\epsilon}.
\]
\end{assumption}

The time-dependent learning rate excess risk is given below. 

\begin{theorem}[Time infinity risk values for SLD class]\label{thm:eventualrisk_SLD} Suppose Assumption~\ref{flat_ass_SLD} holds for the Volterra SLD class and the integrated learning rate satisfies $\Gamma(t) \to \infty$ and $\gamma(t) \to \gamma \in [0, \infty)$. Let the limiting learning rate value $\gamma$ be chosen such that the kernel norm is less than $1$, that is, 
\begin{equation}\label{eq:knorm_SLD}
\mathcal{I}(\gamma) \defas 
\int_0^\infty K(t, s; \AA^T \AA)\,\dif t
=
\frac{\gamma}{2}
\tr
\biggl(
\mathscr{M}(\AA^T \AA)
\bigl( \AA^T \AA + \delta \II_d)\bigr)^{-1}
\biggr) < 1.
\end{equation}

Then with $\Psi_\infty$ given by the limiting empirical risk,
\begin{equation} \label{eq:limit_loss_main}
\Psi_\infty = \left ( 1- \frac{\gamma}{2} \tr \bigg \{ \frac{\mathcal{M} \AA^T\AA}{\AA^T \AA+ \delta \II_d} \bigg \} \right )^{-1} \times \left ( \mathcal{L}(\gf_{\infty}) + \frac{\gamma}{2} \tr \bigg \{ \frac{\mathcal{A} \AA^T\AA}{\AA^T \AA+ \delta \II_d} \bigg \} \right ), 
\end{equation}
the excess risk converges to
\[
\Omega_t-\mathcal{R}\bigl( 
\gf_{\Gamma(t)}\bigr)
\to 
\frac{\gamma}{2} \times
\tr
\bigg \{
\frac{(\mathscr{A}+\mathscr{M}\Psi_\infty)(\nabla^2 \mathcal{R})}{ \AA^T \AA + \delta \II_d}
\bigg \}.
\]
\end{theorem}

\begin{proof}First suppose that the limiting loss value of $\Psi_t$, defined in \eqref{eq:VLoss}, is bounded and it exists at infinity. We show under this condition on $\Psi_t$ that the limiting risk value holds for $\Omega_{\infty}$, defined in \eqref{eq:PLoss}. A simple computation with a change of variables gives
{\small \begin{equation}
\begin{aligned}
&\lim_{t \to \infty}  \Omega_t - \mathcal{R}(\bm{\mathscr{X}}_{\Gamma(t)}^{\text{gf}})\\
&= \lim_{t \to \infty}  \int_0^t \gamma^2(s) \tr \bigg (\nabla^2 \mathcal{R}) \mathcal{A} \exp \big (-2(\AA^T \AA + \delta \II_d)(\Gamma(t)-\Gamma(s)) \big ) \bigg )\, \dif s\\
& \qquad + \lim_{t \to \infty}  \int_0^t \gamma^2(s) \text{tr}\bigg ( (\nabla^2 \mathcal{R}) \mathcal{M} \exp \big (-2(\AA^T \AA + \delta \II_d)(\Gamma(t)-\Gamma(s)) \big ) \bigg ) \Psi_s \, \dif s\\
&= \lim_{t \to \infty}  \int_0^{\Gamma(t)} \gamma(s) \tr \bigg (\nabla^2 \mathcal{R}) \mathcal{A} \exp \big (-2(\AA^T \AA + \delta \II_d)(\Gamma(t)-s) \big ) \bigg )\, \dif s\\
& \qquad + \lim_{t \to \infty} \int_0^{\Gamma(t)} \gamma(s) \text{tr}\bigg ( (\nabla^2 \mathcal{R}) \mathcal{M} \exp \big (-2(\AA^T \AA + \delta \II_d)(\Gamma(t)-s) \big ) \bigg ) \Psi_{\Gamma^{-1}(s)} \, \dif s\\
&= \lim_{t \to \infty}  \int_0^{\Gamma(t)} \gamma(\Gamma(t)-v) \tr \bigg (\nabla^2 \mathcal{R}) \mathcal{A} \exp \big (-2(\AA^T \AA + \delta \II_d)v \big ) \bigg )\, \dif s\\
&\qquad + \lim_{t \to \infty} \int_0^{\Gamma(t)} \gamma(\Gamma(t)-v) \text{tr}\bigg ( (\nabla^2 \mathcal{R}) \mathcal{M} \exp \big (-2(\AA^T \AA + \delta \II_d)v \big ) \bigg ) \Psi_{\Gamma^{-1}(\Gamma(t)-v)} \, \dif v.
\end{aligned}
\end{equation} }
Dominated convergence theorem allows us to interchange the integral and limit as $\Psi_t$ and $\gamma(t)$ are bounded. We pull out the limiting values of $\lim_{t \to \infty} \gamma(t) = \gamma$ and $\Psi_{\infty}$. By integrating, we deduce 
{\small
\begin{equation}
\begin{aligned}
    &\lim_{t \to \infty} \Omega_t - \mathcal{R}(\bm{\mathscr{X}}_{\Gamma(t)}^{\text{gf}})\\
    &= \lim_{t \to \infty}  \int_0^{\Gamma(t)} \gamma(\Gamma(t)-v) \tr \bigg (\nabla^2 \mathcal{R}) \mathcal{A} \exp \big (-2(\AA^T \AA + \delta \II_d)v \big ) \bigg )\, \dif s\\
    & \qquad + \lim_{t \to \infty} \int_0^{\Gamma(t)} \gamma(\Gamma(t)-v) \text{tr}\bigg ( (\nabla^2 \mathcal{R}) \mathcal{M} \exp \big (-2(\AA^T \AA + \delta \II_d)v \big ) \bigg ) \Psi_{\Gamma^{-1}(\Gamma(t)-v)} \, \dif v\\
    &= \gamma \int_0^{\infty} \text{tr}\bigg ( (\nabla^2 \mathcal{R}) \mathcal{A} \exp \big (-2(\AA^T \AA + \delta \II_d)v \big ) \bigg ) \, \dif v \\
    & \qquad + \gamma \Psi_{\infty} \int_0^{\infty} \text{tr}\bigg ( (\nabla^2 \mathcal{R}) \mathcal{M} \exp \big (-2(\AA^T \AA + \delta \II_d)v \big ) \bigg ) \, \dif v\\
    &=  \gamma  \tr \bigg ( (\nabla^2 \mathcal{R}) \big ( \mathcal{M}\Psi_{\infty} + \mathcal{A} \big ) (2(\AA^T \AA + \delta \II_d))^{-1} \bigg ).
\end{aligned}
\end{equation}
}
The result for the limiting risk  value $\lim_{t \to \infty} \Omega_t - \mathcal{R}(\bm{\mathscr{X}}_{\Gamma(t)}^{\text{gf}})$ follows. 

It remains to show that $\Psi_t$ is bounded and exists at infinity with its limiting value given by \eqref{eq:limit_loss_main}. Recall the loss kernel for $\Psi_t$ given by
\begin{equation}
    K(t,s) \defas K(t,s; \AA^T\AA) = \gamma^2(s) \tr\bigg (  \mathcal{M} \AA^T\AA \exp \big (-2(\AA^T\AA + \delta\II_d)(\Gamma(t)-\Gamma(s)) \big) \bigg ),
\end{equation}
so that $\Psi_t$ is the solution to the Volterra equation
\begin{align} \label{eq:psi_volterra}
    \Psi_t = \mathscr{L}(\bm{\mathscr{X}}_{\Gamma(t)}^{\text{gf}}) + \int_0^t K(t,s) \Psi_s \, \dif s.
\end{align}
Under the kernel norm bounded by $1$, \eqref{eq:knorm_SLD}, we show that the kernel $K(s,t)$ is of $L^{\infty}$-type on $[0, \infty)$. A kernel is $L^\infty$-type if $\vertiii{K}_{L^{\infty}(J)} < \infty$ for a set $J \subset \mathbb{R}$ where $\vertiii{K}_{L^{\infty}(J)} = \sup_{t \in J} \int_J |K(s,t)| \, \dif s$ \citep[Chapter 9.2]{gripenberg1980volterra}. For this, we see that for each $t$ and $s$
\begin{equation} \label{eq:change_variables_volterra}
\begin{aligned}
K(t,s) \le \widehat{\gamma} \cdot \gamma(s) \tr\bigg (  \mathcal{M} \AA^T\AA \exp \big (-2(\AA^T\AA + \delta\II_d)(\Gamma(t)-\Gamma(s)) \big) \bigg ).
\end{aligned}
\end{equation}
This implies by change of variables that
\begin{equation}
    \begin{aligned}
    \int_0^t K(t,s) \, \dif s &\le \int_0^{\Gamma(t)} \widehat{\gamma} \tr\bigg (  \mathcal{M} \AA^T\AA \exp \big (-2(\AA^T\AA + \delta\II_d)(\Gamma(t)-s) \big) \bigg ) \dif s\\
    &\le \frac{\widehat{\gamma}}{2} \tr \big ( \mathcal{M} \AA^T\AA (\AA^T \AA + \delta\II_d)^{-1} \big ) < \infty. 
    \end{aligned}
\end{equation}
Hence, it follows that the kernel $K$ is $L^\infty$-type on $[0,\infty)$. To prove the boundedness assumption of $\Psi_t$, we will need something slightly stronger. We show that there exists a finite number of intervals $J_i$ such that $\cup_i J_i = [0, \infty)$ and $\vertiii{K}_{L^{\infty}(J_i)} \le 1$. From this and Theorem 9.3.13 in \citep{gripenberg1980volterra}, it will follow that the resolvent is also of type $L^\infty$ on $[0, \infty)$. Since $\gamma(t) \to \gamma$, there exists a $t_0$ such that for all $t \ge t_0$, $\gamma(t) \le \gamma + \varepsilon$. This $\varepsilon > 0$ can be chosen sufficiently small such that $\gamma + \varepsilon < 2\big (\tr( \mathcal{M} \AA^T \AA (\AA^T\AA + \delta \II_d)^{-1} ) \big )^{-1}$ (see \eqref{eq:knorm_SLD} which gives an upper bound on $\gamma$). First, we observe that 
\[\sup_{t \ge 0} \sup_{0 \le s \le t_0} K(t,s) \le \widehat{\gamma}^2 \tr \big ( \mathcal{M} (\AA^T\AA) e^{2(\AA^T \AA + \delta \II_d)\Gamma(t_0)} \big ) < \infty.\]
We break up the interval $[0, t_0]$ into finitely many intervals of length each of which has a length strictly less than $\big (\widehat{\gamma}^2 \tr \big ( \mathcal{M} \AA^T\AA e^{2(\AA^T \AA + \delta \II_d)\Gamma(t_0)} \big ) \big )^{-1}$. If we denote these intervals by $J_i$, then it immediately follows by bounding the integral using the sup of $K$ multiplied by the length of the interval $J_i$ that 
\[ \vertiii{K}_{L^{\infty}(J_i)} = \sup_{t \in J_i} \int_{J_i} K(t,s) \, \dif s < 1. \]
It only remains to show on the tail, that is, $J_{\infty} \defas (t_0, \infty)$, for which $\vertiii{K}_{L^{\infty}(J_{\infty})} < 1$. Using the same change of variables as in \eqref{eq:change_variables_volterra} and our choice of $t_0$, we have that for all $t \ge t_0$
\begin{align*}
    \int_{t_0}^t K(t,s) \, \dif s 
    &\le \int_{t_0}^t (\gamma + \varepsilon) \gamma(s) \tr\bigg ( \mathcal{M} \AA^T \AA \exp \big (-2(\AA^T\AA + \delta\II_d)(\Gamma(t)-\Gamma(s)) \big) \bigg ) \, \dif s\\
    &= \int_{\Gamma(t_0)}^{\Gamma(t)} (\gamma + \varepsilon) \tr\bigg ( \mathcal{M} \AA^T \AA \exp \big (-2(\AA^T\AA + \delta\II_d)(\Gamma(t)-s) \big) \bigg ) \dif s\\
    &\le \frac{\gamma + \varepsilon}{2} \tr \big ( \mathcal{M} \AA^T\AA (\AA^T \AA + \delta\II_d)^{-1} \big ) < 1.
\end{align*}
The last inequality following by our assumption on $\gamma + \varepsilon$ being sufficiently small. By Theorem 9.3.13 in \citep{gripenberg1980volterra}, we have that the resolvent is also of type $L^\infty$ on $[0, \infty)$. We also have that $K(t,s)$ is of bounded type, that is the kernel is bounded (see \citep[Definition 9.5.2]{gripenberg1980volterra} for precise definition). Since the forcing term $\mathscr{L}(\bm{\mathscr{X}}_{\Gamma(t)}^{\text{gf}})$ and $\int_0^t A(t,s; \AA^T\AA) \, \dif s$ are bounded, then it follows by \citep[Theorem 9.5.4]{gripenberg1980volterra} that the solution to the Volterra equation \eqref{eq:psi_volterra}, $\Psi_t$, is bounded. 

We now show that $\Psi_t$ exists at infinity. Fix a $\varepsilon > 0$. By the assumptions on the learning rate, there exists a $t_0 > 0$ such that for all sufficiently large $t \ge s \ge t_0$
\begin{equation}
    \begin{aligned}
    \gamma- \varepsilon  \le \gamma(t) \le \gamma + \varepsilon \quad \text{and} \quad (\gamma - \varepsilon)(t-s) \le \Gamma(t) - \Gamma(s) \le (\gamma+ \varepsilon) (t-s). 
    \end{aligned}
\end{equation}
Using these inequalities for $\gamma(t)$, we get an upper bound and lower bound on the kernel $K(t,s)$ which we denote by $\overline{K}(t,s)$ and $\underline{K}(t,s)$, respectively. Specifically for all $t, s \ge t_0$, 
\begin{equation}
    \begin{aligned} \label{eq:kernel_bound_1}
    K(t,s) &\le \overline{K}(t,s) \defas (\gamma+ \varepsilon)^2 \tr\bigg ( \mathcal{M} \AA^T \AA \exp \big (-2(\AA^T\AA + \delta\II_d)(\gamma-\varepsilon)(t-s) \big) \bigg )\\
    K(t,s) & \ge \underline{K}(t,s)\defas (\gamma- \varepsilon)^2 \tr\bigg ( \mathcal{M} \AA^T \AA \exp \big (-2(\AA^T\AA + \delta\II_d)(\gamma+\varepsilon)(t-s) \big) \bigg ).
    \end{aligned}
\end{equation}
The kernels $\overline{K}(t,s)$ and $\underline{K}(t,s)$ are substantially nicer than the original $K(t,s)$ because they are proper convolution kernels. Here one can define $\overline{K} \, : \, [0, \infty) \to \mathbb{R}$ by
\[
\overline{K}(t) \defas  (\gamma + \varepsilon)^2 \tr\bigg ( \mathcal{M} \AA^T \AA \exp \big (-2(\AA^T\AA + \delta\II_d)(\gamma-\varepsilon)t \big) \bigg ).
\]
Then it follows that $\overline{K}(t,s) = \overline{K}(t-s)$. A similar result holds for $\underline{K}(t,s)$. 

For ease of notation, define the forcing function: for $t \ge t_0$
\begin{equation}
    F(t) \defas \mathcal{L}(\bm{\mathscr{X}}_{\Gamma(t)}^{\text{gf}}) + \int_0^{t_0} K(t,s) \Psi_s \, \dif s + \int_0^{t} A(t,s; \AA^T \AA) \, \dif s,
\end{equation}
where $\Psi_t$ is a solution to \eqref{eq:psi_volterra}. Similar to the definitions of $\overline{K}(t)$ and $\underline{K}(t)$, we define $\overline{F}(t)$ and $\underline{F}(t)$ respectively as
\begin{equation}
    \underline{F}(t) \le F(t) \le \overline{F}(t),
\end{equation}
\begin{equation}
\begin{aligned}
    \text{where} \quad & \overline{F}(t) \defas  \mathcal{L}(\bm{\mathscr{X}}_{\Gamma(t)}^{\text{gf}}) + \int_0^{t_0} K(t,s) \Psi_s \, \dif s + \int_0^{t_0} A(t,s; \AA^T \AA) \, \dif s\\
    & \qquad + \int_0^t (\gamma + \varepsilon)^2 \tr\bigg ( \mathcal{A} \AA^T \AA \exp \big (-2(\AA^T\AA + \delta\II_d)(\gamma-\varepsilon)(t-s) \big) \bigg )\\
    \text{and} \quad & \underline{F}(t) \defas \mathcal{L}(\bm{\mathscr{X}}_{\Gamma(t)}^{\text{gf}}) + \int_0^{t_0} K(t,s) \Psi_s \, \dif s + \int_0^{t_0} A(t,s; \AA^T \AA) \, \dif s\\
    & \qquad + \int_0^t (\gamma- \varepsilon)^2 \tr\bigg ( \mathcal{A} \AA^T \AA \exp \big (-2(\AA^T\AA + \delta\II_d)(\gamma+\varepsilon)(t-s) \big) \bigg ) \, \dif s
    \end{aligned}
\end{equation}

Because $\Psi_s$ is bounded, it follows that $\lim_{t \to \infty} \int_0^{t_0} K(t,s) \Psi_s = \lim_{t \to \infty} A(t, s; \AA^T \AA) = 0$. Also it is clear that the $F(t), \overline{F}(t)$, and $\underline{F}(t)$ are bounded. 

Using the upper/lower bound on the kernel \eqref{eq:kernel_bound_1}, we can squeeze the value of $\Psi_t$ between two expressions: for $t, s\ge t_0$,
\begin{equation}
\begin{aligned} \label{eq:gronwall}
\underline{F}(t) + \int_{t_0}^t \underline{K}(t,s) \Psi_s \, \dif s \le \Psi_t \le \overline{F}(t) + \int_{t_0}^t \overline{K}(t,s) \Psi_s \, \dif s.
\end{aligned}
\end{equation}
 Using a similar argument for $K(t,s)$ and choosing $\varepsilon$ sufficiently small, $\overline{K}(t,s)$ and $\underline{K}(t,s)$ are $L^{\infty}$-type on $[0, \infty)$. Moreover using a similar argument as we did for $K$ itself, the norms $\vertiii{\overline{K}}_{L^\infty([0,\infty))} < 1$ and  $\vertiii{\underline{K}}_{L^\infty([0,\infty))} < 1$. Here we used the upper bound on $\gamma$ in \eqref{eq:knorm_SLD} and a sufficiently small $\varepsilon$. Note we do not need to break up into finite intervals. As before, the resolvent then is of $L^{\infty}$-type on $[0, \infty)$ \citep[Corollary 9.3.10]{gripenberg1980volterra}. Further because of non-negativity, Proposition 9.8.1 in \citep{gripenberg1980volterra} yields that the resolvents are also non-negative. 
 
 Consider the upper bound (a similar argument will hold for the lower bound). We can apply Gronwall's inequality \eqref{eq:gronwall} \citep[Theorem 9.8.2]{gripenberg1980volterra}. It follows that $\Psi_t$ is upper bounded (lower bounded) by the solutions $\overline{\Psi}_t$ ($\underline{\Psi}_t$) to the following convolution Volterra equations
\[ \overline{\Psi_t} = \overline{F}(t) + \int_{t_0}^t \overline{K}(t,s) \overline{\Psi}_s \, \dif s \quad \text{and} \quad \underline{\Psi}_t = \underline{F}(t) + \int_{t_0}^t \underline{K}(t,s) \underline{\Psi}_s \, \dif s.\]
Specifically, we have $\underline{\Psi}_t \le \Psi_t \le \overline{\Psi}_t$ for all $t \ge t_0$. Since $\overline{\Psi}_t$ and $\underline{\Psi}_t$ are solutions to a proper convolution-type Volterra equation and both functions $\overline{F}(t)$, and $\underline{F}(t)$ have limits at infinity ($\overline{F}(\infty) \defas \lim_{t\to\infty} \overline{F}(t)$ and $\underline{F}(\infty) \defas \lim_{t\to\infty} \underline{F}(t)$), by \citep{Asmussen}, for $t \ge t_0$
{\small \begin{equation}
    \limsup_{t \to \infty} \Psi_t \le \limsup_{t \to \infty} \overline{\Psi}_t = \overline{F}(\infty) 
    \big (1- \vertiii{ \overline{K}}_{L^{\infty}([t_0, \infty))} \big )^{-1} \le \overline{F}(\infty) \big (1- \vertiii{ \overline{K}}_{L^{\infty}([0, \infty))} \big )^{-1},
\end{equation}
}
and similarly, the lower bound gives
\begin{equation}
    \liminf_{t \to \infty} \Psi_t \ge \liminf_{t \to \infty} \underline{\Psi}_t  \le \underline{F}(\infty) \big (1- \vertiii{ \underline{K}}_{L^{\infty}([0, \infty))} \big )^{-1} .
\end{equation}
A simple computation yields that 
\begin{equation}
    \begin{gathered}
    \vertiii{ \overline{K}}_{L^{\infty}([0, \infty))} = \frac{(\gamma+\varepsilon)^2}{\gamma-\varepsilon} G(\mathcal{M}) \quad \text{and} \quad  \vertiii{ \underline{K}}_{L^{\infty}([0, \infty))} = \frac{(\gamma-\varepsilon)^2}{\gamma+\varepsilon} G(\mathcal{M})\\
    \text{and} \qquad \overline{F}(\infty) = \mathcal{L}(\gf_{\infty}) + \frac{(\gamma+\varepsilon)^2}{\gamma-\varepsilon} G(\mathcal{A}) \quad \text{and} \quad \underline{F}(\infty) = \mathcal{L}(\gf_{\infty}) + \frac{(\gamma-\varepsilon)^2}{\gamma+\varepsilon} G(\mathcal{A})\\
    \text{where} \quad G(\mathcal{H}) \defas \tfrac{1}{2} \tr \big ( \mathcal{H} \AA^T\AA (\AA^T \AA+ \delta \II_d)^{-1} \big ).
    \end{gathered}
\end{equation}
So for any sufficiently small $\varepsilon > 0$, we have that 
{\small \begin{equation} \begin{aligned}
   \left ( 1 - \frac{(\gamma-\varepsilon)^2}{\gamma+\varepsilon} G(\mathcal{M}) \right )^{-1} \times\left \{ \mathcal{L}(\gf_{\infty}) + \frac{(\gamma-\varepsilon)^2}{\gamma+\varepsilon} G(\mathcal{A}) \right \} &\le \liminf_{t \to \infty} \Psi_t \\
   \le \limsup_{t \to \infty} \Psi_t \le \left ( 1- \frac{(\gamma+\varepsilon)^2}{\gamma-\varepsilon} G(\mathcal{M}) \right )^{-1} &\times \left \{ \mathcal{L}(\gf_{\infty}) + \frac{(\gamma+\varepsilon)^2}{\gamma-\varepsilon} G(\mathcal{A}) \right \}.
   \end{aligned}
\end{equation}
}

As this holds for any sufficiently small $\varepsilon$, the result follows by sending $\varepsilon \to 0$.
\end{proof}




\section{Algorithmic regularization} \label{sec:app_algorithmic}
In this section, we discuss the exact asymptotic convergence rates for SGD and full batch momentum algorithms on high-dimensional $\ell^2$-regularized least squares problems. The results in this section (e.g., Theorems~\ref{thm:sgdrates} and \ref{thm:mgdrates}) were shown in a series of papers \citep{paquette2021dynamics, paquetteSGD2021,paquette2020halting} that explored exact trajectories of loss function.  

\subsection{Convergence rates of SGD} \label{sec:convergence_rate}

To characterize the rates, we define 
\(
\lambda_{\min}
\)
as the smallest non-zero eigenvalue of $\AA^T\AA$.  Then for generic initial conditions, (in particular almost surely if $\XX_0$ is isotropic norm $1$), then
\[
\lim_{t \to \infty} \biggl(
\mathscr{L}(\gf_t)
-
\mathscr{L}(\gf_\infty)
\biggr)^{1/t} = 
\begin{cases}
e^{-\gamma(\lambda_{\min} + \delta)}, & \text{if } \delta > 0, \\
e^{-2\gamma \lambda_{\min}}, & \text{otherwise}.
\end{cases}
\]
The rate of convergence of $\Psi_t$ to $\Psi_\infty$ is given by, (for small $\gamma$), the rate above.  For larger $\gamma$, another rate can frustrate the convergence.  Recall the \emph{Malthusian exponent} of the convolution Volterra equation in \eqref{eq:Malthusian} is given by
{\small 
    \begin{equation}
    \lambda_* = 
    \inf\biggl\{
    x : 
    1
    =\int_0^\infty \! \! \! \! \! e^{ xt}K(t; \AA^T\AA)\,\dif t 
    =\gamma^2 
    \int_0^\infty \! \! \! \! \!
    e^{ xt}
    \tr
    \biggl(
    \mathscr{M}
    \AA^T\AA
    \exp\bigl( -2\gamma(\AA^T \AA + \delta \II )t\bigr)
    \biggr)\,\dif t 
    \biggr\}.
    \end{equation}
}
The set may be empty, in which case the infimum is $\infty$.  We recall below Theorem~\ref{thm:sgdrates}.

\begin{theorem}\label{thm:sgdrates_1}
For $\gamma>0$ satisfying $\mathcal{I}(\gamma) < 1$ (see \eqref{eq:knorm}), define
\begin{equation}\label{eq:ratefunction_1}
\Xi(\gamma)
\defas
\begin{cases}
\min\{
\gamma(\lambda_{\min}+\delta),
\lambda_*(\gamma)
\}
& \text{ if } \delta > 0, \\
\min\{
2\gamma
\lambda_{\min},
\lambda_*(\gamma)
\}
& \text{ if } \delta = 0.
\end{cases}
\end{equation}
Then the rates of convergence of both the training and test loss are
\[
\lim_{t \to \infty} \bigl(
\Psi_t
-
\Psi_\infty
\bigr)^{1/t} = e^{-\Xi(\gamma)}
=
\lim_{t \to \infty} \bigl(
\Omega_t
-
\Omega_\infty
\bigr)^{1/t}
\]
Furthermore, when $\gamma = n/\tr(\AA^T\AA)$, we have the rate guarantee $\Xi(\gamma) \geq \tfrac{\lambda_{\min} n}{2\tr(\AA^T\AA)}.$
\end{theorem}

\begin{proof} See \citep[Theorem 1.2]{paquetteSGD2021} for proof.  
\end{proof}

\subsection{Momentum GD (M-GD) rates}
\label{sec:momentum}
In this section, we consider a popular \textit{deterministic} or \textit{full-batch} algorithm for solving the ridge regression problem in \eqref{eq:rr}, that is, gradient descent with momentum (a.k.a Polyak momentum). Throughout this section, we use the notation, $\mgd_t = \xx_t$. Gradient descent with momentum (M-GD), initialized at $\xx_0 \in \mathbb{R}^d$ and $\xx_1 = \xx_0 - \frac{\gamma}{1+m} \nabla f(\xx_0)$, iterates for $k \ge 1$
\begin{equation} \label{eq:GDM_recurrence}
    \begin{aligned}
    \xx_{k+1} = \xx_k + m(\xx_k-\xx_{k-1}) - \gamma \nabla f(\xx_k), 
    \end{aligned}
\end{equation}
where $\gamma, m > 0$ are the stepsize and momentum parameters respectively. From Proposition 3.1 in
\citep{paquette2020halting}, there exists $k$-degree polynomials $P_k$ and $Q_k$ such that the iterates of GD+M satisfy the following 
\begin{equation}
    \xx_k = P_k(\AA^T\AA + \delta \II) \xx_0 + Q_k(\AA^T\AA + \delta \II) \AA^T\bb, \quad \text{with $P_k(\lambda) = 1-(\lambda) Q_k(\lambda)$}
\end{equation}
and the coefficients of $P_k$ and $Q_k$ only depend on the largest and smallest eigenvalue of $\AA^T\AA$. For Polyak, similar to the work in \citep[Section 3.1]{paquette2020halting}, we can give an explicit representation for these polynomials $P_k$ and $Q_k$. 

\begin{proposition}[Polynomial representation of M-GD] Suppose $\xx_0 \in \mathbb{R}^d$ and fix a stepsize $\gamma > 0$ and momentum parameter $m > 0$. For the iterates of GD+M on \eqref{eq:rr} with ridge parameter $\delta > 0$, we have the following representation for the polynomials 
\begin{equation}
    \xx_k = P_k(\AA^T\AA + \delta \II) \xx_0 + Q_k(\AA^T \AA + \delta \II) \AA^T \bb, 
\end{equation}
where $P_k$ and $Q_k$ are $k$-degree polynomials satisfying
\begin{equation}
    \begin{gathered} \label{eq:polyak_poly}
    P_k(\lambda) = m^{k/2} \left ( \frac{2m}{1+m} T_k (\sigma(\lambda)) + \left (1 - \frac{2m}{1+m} \right ) U_k(\sigma(\lambda) \big ) \right ) \quad \text{and} \quad Q_k(\lambda) = \frac{1-P_k(\lambda)}{\lambda}\\
    \text{where} \quad \sigma(\lambda) = \frac{1 + m - \gamma \lambda}{2 \sqrt{m}}\\
    \text{ and $T_k$, $U_k$ are Chebyshev polynomials of the 1st and 2nd kind respectively.}
    \end{gathered}
\end{equation}
\end{proposition}
\begin{proof} The proof can be found in \citep[Appendix A.2]{paquette2020halting} or \citep[Chapter 11] {ginsburg1959regined}. We include a sketch of the proof. From the recurrence in \eqref{eq:GDM_recurrence} and the gradient of the ridge regression, the polynomials $P_k$ that generate GD+M satisfy the following three-term recurrence, 
\begin{equation}
    \begin{aligned} \label{eq:GD_M_polynomial}
        P_{k+1}(\lambda) &= (1+m-\gamma \lambda) P_k(\lambda) - m P_{k-1}(\lambda)\\
        P_k(\lambda) &= \frac{P_{k+1}(\lambda) + m P_k(\lambda)}{1+m-\gamma \lambda}.
    \end{aligned}
\end{equation}
We define the polynomial generating function for $P_k$ as $\mathcal{G}(\lambda, t) = \sum_{k=0}^\infty t^k P_k(\lambda)$. Using the recurrence in \eqref{eq:GD_M_polynomial}, we get that 
\begin{align*}
    \mathcal{G}(\lambda, t) &= 1 + \frac{1}{t(1-m+\gamma \lambda)}\sum_{k=2}^\infty t^k P_k(\lambda) - \frac{mt}{1-m+\gamma \lambda} \sum_{k=0}^\infty t^k P_k(\lambda)\\ 
    &= 1 + \frac{1}{t(1-m + \gamma \lambda)} \big [ \mathcal{G}(\lambda, t) - 1 - t(1 - \tfrac{\gamma}{1+m} \lambda) \big ]  - \frac{mt}{1-m + \gamma \lambda} \mathcal{G}(\lambda, t).
\end{align*}
By solving this expression for the generating polynomial, we have
\[\mathcal{G}(\lambda, t) = \frac{1 + t(m-(\gamma + \frac{\gamma}{1+m}))}{1 - t(1-m + \gamma \lambda) - m t^2}. \]
This generating function for M-GD closely resembles the generating function for Chebyshev polynomials of the 1st and 2nd kind. Under simple transformations (e.g., $t \mapsto \frac{t}{\sqrt{m}}$), this is exactly the case. These transformations yield the expression in \eqref{eq:polyak_poly}. 
\end{proof}

The role of $\sigma(\lambda)$ is to transform the eigenvalues of $\AA^T\AA + \delta \II_d$ within a specific range controlled by the learning rate and momentum. It is known that the Chebyshev polynomials are well-behaved on the interval of $[-1, 1]$ and grow exponentially off of this region. 

Moreover for a generic quadratic applied to $\xx_k$, the rate of convergence will be controlled by $P_k(\lambda)$. Using standard asymptotic behavior of Chebyshev polynomials, we can derive asymptotic rates based on $\lambda_{\min} \defas \lambda_{\min}(\AA^T\AA + \delta \II_d)$ and $\lambda_{\max} \defas \lambda_{\max}(\AA^T\AA + \delta \II_d)$, the smallest (non-zero) and largest eigenvalues of $\AA^T \AA$ respectively. We record this result below

\begin{proposition}[Asymptotic rates of M-GD]
The asymptotic rate of M-GD is 
{\small 
\begin{equation}
\limsup_{k \to \infty} \sqrt[k]{P_k} = \begin{cases}
\sqrt{m} & \text{if $\gamma \in \left [ \frac{(1-\sqrt{m})^2}{\lambda_{\min}}, \frac{(1+\sqrt{m})^2}{\lambda_{\max}} \right ]$}\\
\sqrt{m} \big ( | \sigma(\lambda_{\min})| + \sqrt{\sigma(\lambda_{\min})^2 -1} \big )& \text{if $\gamma \in \big [0, \min\big \{\frac{2(1+m)}{\lambda_{\min} + \lambda_{\max}}, \frac{(1-\sqrt{m})^2}{\lambda_{\min}} \big \} \big ]$ }\\
\sqrt{m} \big ( | \sigma(\lambda_{\max})| + \sqrt{\sigma(\lambda_{\max})^2 -1} \big ) & \text{if $\gamma \in \big [ \max \big \{\frac{2(1+m)}{\lambda_{\min} + \lambda_{\max}}, \frac{(1+\sqrt{m})^2}{\lambda_{\max}} \big \}, \frac{2(1+m)}{\lambda_{\max}} \big ]$ }\\
\ge 1  & \text{otherwise}.
\end{cases}
\end{equation}
}
\end{proposition}

\begin{proof} See \citep{pedregosa2021residual} for a complete proof. The result follows from knowing that the iterates are given by Chebyshev polynomials and then applying well-known asymptotics of Chebyshev polynomials to get the convergence rate. 
\end{proof}

We can minimize over the rate to find the optimal parameters. In this case, they become the parameters used in the Heavy-Ball algorithm \citep{Polyak1962Some} where 
\begin{equation}
    \begin{gathered}
    m = \left ( \frac{\sqrt{\lambda_{\max}} - \sqrt{\lambda_{\min}}}{\sqrt{\lambda_{\max}} + \sqrt{\lambda_{\min}}} \right )^2 \quad \text{and} \quad \gamma = \left ( \frac{2}{\sqrt{\lambda_{\max}} + \sqrt{\lambda_{\min}} } \right )^2.
    \end{gathered}
\end{equation}
A simple computation yields that the asymptotic rate for Heavy-Ball is $\frac{\sqrt{\lambda_{\max}} - \sqrt{\lambda_{\min}}}{\sqrt{\lambda_{\max}} + \sqrt{\lambda_{\min}}}$.

\begin{figure}
    \centering
           \includegraphics[scale = 0.18]{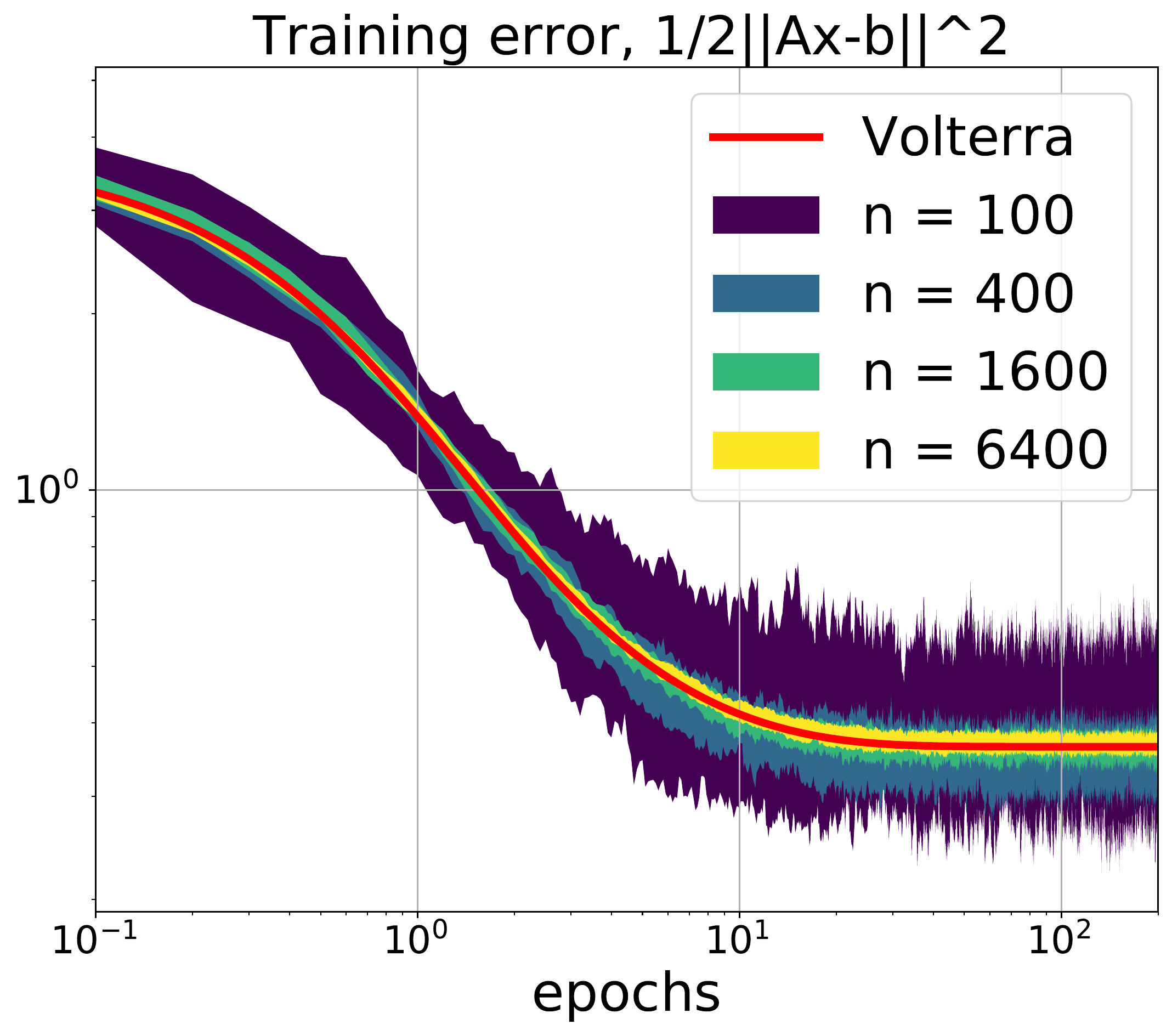}
       \includegraphics[scale = 0.18]{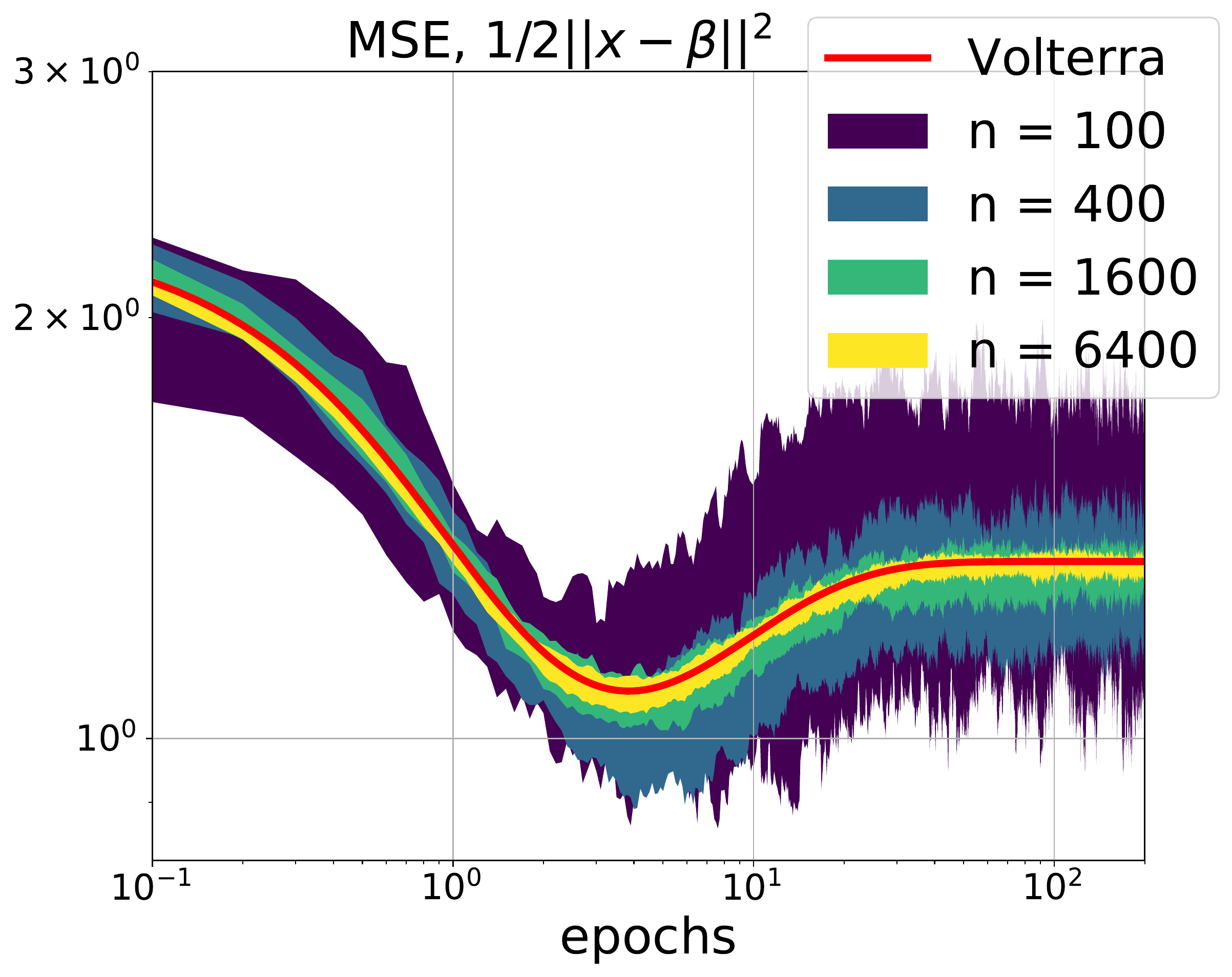}
       \quad \includegraphics[scale = 0.18]{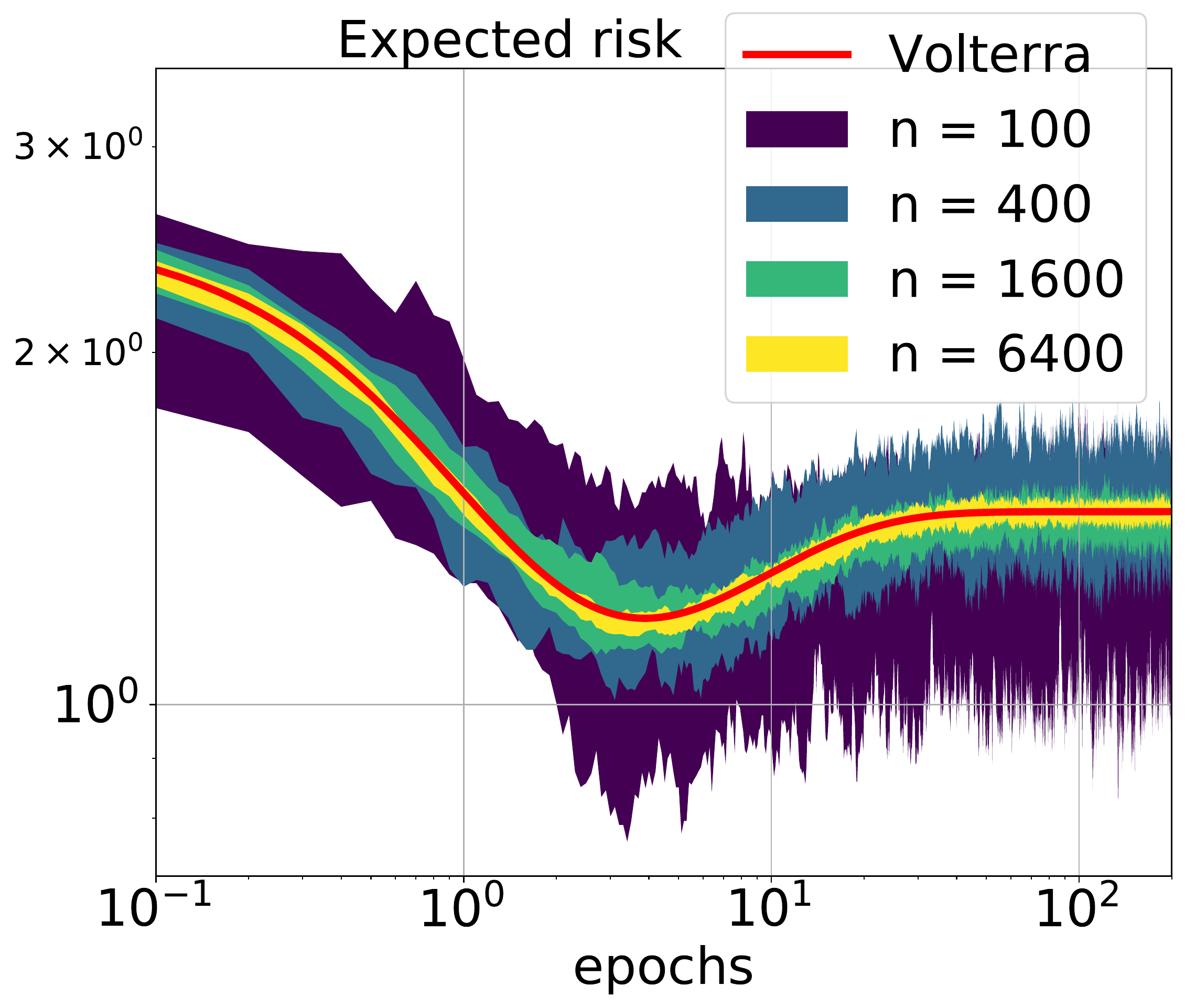}
     \caption{\textbf{Concentration of mean squared error (MSE) and expected test loss, $\tfrac{1}{2}\|\xx-\bbeta\|_2^2$, for SGD} on a Gaussian random $\ell^2$-regularized least-squares problem (Section~\ref{sec:quasi_random}) where $\bbeta \sim N(\bm{0}, \II_d)$ is the ground truth signal and a generative model $\bb =  \AA \bbeta + \eeta$ where entries of $\eeta$ iid standard normal with $\|\eeta\|_2^2 = 2.25$, $n = 0.9d$ with $\ell^2$-regularization parameter $\delta = 0.1$. SGD with constant learning rate $\gamma  =0.8$ was initialized at $\xx_0 \sim N(\bm{0}, 4 \II_d)$ (independent of $\AA$, $\bbeta$); an $80\%$ confidence interval (shaded region) over $10$ runs for each $n$. Any quadratic statistic, such as the MSE, becomes non-random in the large limit and all runs of SGD converge to a deterministic function $\Omega_t$ \textcolor{red}{(red)} solving a Volterra equation \eqref{eqa:PLoss}. This is an illustration of Theorem~\ref{thm:lppp} and Theorem~\ref{thm:trainrisk}.
     }
    \label{fig:MSE_Gaussian}
\end{figure}

\section{Numerical simulations} \label{sec:numerical_simulations}

To illustrate our theoretical results and conjectures we report simulations and experiments using SGD with constant learning rate on the $\ell^2$-regularized least squares problem. In all simulations for the random $\ell^2$-regularized least-square problem, the vectors $\eeta$, and $\bbeta$ are sampled from a standard Gaussian and the initialization vector $\xx_0 = \bm{0}$ (for Figures~\ref{fig:CIFAR_5M_streaming} and \ref{fig:gaussianstreaming}) and $N(0, 4 \II_d)$ (Figure~\ref{fig:MSE_Gaussian}). For the random features model (see Section~\ref{sec: motivating_applications} and Figure~\ref{fig:CIFAR_5M_streaming}, a standardized ReLu activation function was applied, that is 
\begin{equation} \label{eq:standarized_relu}
\sigma(\cdot) = \frac{\max \{\cdot, 0\} - 0.5( \pi)^{-1}}{0.5-0.5 \pi^{-1}}.
\end{equation}
The entries of the hidden weight matrix $\WW \in \mathbb{R}^{n_0 \times d}$ in the random feature model are standard normal.

\paragraph{Volterra: computing theoretical dynamics.} When the entries of $\AA$ are generated by standard Gaussians, a celebrated work \citep{marvcenko1967distribution} gives an explicit limiting density for the eigenvalues when $d$ and $n$ are proportional. In this case, the Volterra equation \eqref{eqa:VLoss} for the loss function $\mathscr{L}$ is computable without needing to input the empirical eigenvalues of the data matrix $\AA \AA^T$. Since the covariance of standard Gaussians is explicitly $\SSigma_f = \II_d$ (see Appendix~\ref{sec: motivating_applications}), one can also directly solve for the expected risk \eqref{eqa:PLoss} for applications such as in-distribution expected risk. As such, the Volterra equation is completely determined. To solve it, a Chebyshev quadrature was used to derive a numerical approximation for the kernel, $K$, \eqref{eqa:VLoss}. The size of the grid points used to compute the numerical integration does effect the Volterra equations convergence to the theoretical limit. We suggest that the number of epochs be equal to the number of grid points used in the numerical quadrature rule. Next, to generate the solution $\mathscr{L}$ of the Volterra equation, we implement a Picard iteration which finds a fix point to the Volterra equation by repeatedly convolving the kernel and adding the forcing term. Despite the numerical approximations to integrals, the resulting solutions to the Volterra equation ($\Psi$ and $\Omega$) model the true behavior of SGD remarkably well. Similarly, by evaluating contour integrals, random features with Gaussian $\XX$ and $\WW$ known explicit formulas for the limiting densities of eigenvalues and eigenvectors (see e.g., \citep{adlam2020neural}). This approach was used to compute the theoretical dynamics in Figure \ref{fig:MSE_Gaussian}.

When the limiting eigenvalues and eigenvectors are unavailable, as in the case of real data sets, an empirical Volterra equation solver was used. We computed the svd of the data matrix $\AA$ and calculated an empirical covariance for $\SSigma_f$ (see Appendix~\ref{sec: motivating_applications}). The singular values and vectors of $\AA$ and $\SSigma_f$ were then used to compute the forcing term (i.e., the GF terms $\mathscr{L}(\gf)$ and $\mathcal{R}(\gf)$) and kernel $K$ \eqref{eqa:V}. As before, a Chebyshev quadrature was used to derive the integral for the kernel $K$ and a Picard iteration to find the fix point of the Volterra was applied. This method was used to compute the theoretical dynamics $\Omega_t$ and $\Psi_t$ in Figures \ref{fig:CIFAR_5M_streaming} and \ref{fig:gaussianstreaming}.

\renewcommand{\arraystretch}{1.1}
\ctable[notespar,
caption = {{\bfseries Summary of the eigenvalues in ICR} with normalized trace equal to 1.0, i.e., $\tfrac{1}{n} \tr( \AA \AA^T ) = 1.0$. All data sets were standardized before applying any transformations (e.g., random features). For random features (RF), standard Gaussian $\WW \in \mathbb{R}^{n_0 \times d}$ applied to the data set followed by entry-wise application of standardized ReLu (see \eqref{eq:standarized_relu} and Appendix~\ref{sec: motivating_applications} and Appendix~\ref{sec:numerical_simulations} for exact set-up). \vspace{-0.5em}
} ,label = {table:ICR_values},
captionskip=2ex,
pos =!t
]{l c c c c}{\tnote[1]{\citep{krizhevsky09learningmultiple}} \tnote[2]{\citep{nakkiran2021bootstrap}} \tnote[3]{\citep{lecun2010mnist}} }{
\toprule
 & & & \multicolumn{2}{c}{\textbf{Eigenvalues of $\AA \AA^T$}} \\
\textbf{Data set} & \textbf{Samples} $(n)$ & \textbf{Features} $(d)$ & \textit{Largest} & \textit{Smallest}\\
\midrule
\begin{minipage}{0.25\textwidth}  CIFAR-10\tmark[1] (all)
\end{minipage} & 50,000 & 3,072  & 11,118.80 & $4.7\cdot 10^{-4}$\\
\midrule
\begin{minipage}{0.25\textwidth}  CIFAR-10\tmark[1] RF\\
{\small large $d$}
\end{minipage} & 50,000 & 5,551 & 8,162.84 & $2.8\cdot10^{-1}$\\
\midrule
\begin{minipage}{0.25\textwidth}  CIFAR-10\tmark[1] RF\\
{\small small $d$}
\end{minipage} & 50,000 & 452 & 8,403.31 & 13.18\\
\midrule
\begin{minipage}{0.25\textwidth}  CIFAR-5m\tmark[2] (all)
\end{minipage} &  5 million & 3,072 & 1,195,595.52 & $1.03 \cdot 10^{-1}$\\
\midrule
\begin{minipage}{0.25\textwidth}  CIFAR-5m\tmark[2] (car/plane)
\end{minipage} &  1 million & 3,072  & 258,599.09 & $1.7 \cdot 10^{-2}$\\
\midrule
\begin{minipage}{0.25\textwidth}  Gaussian\\
{\small \textit{under parameterized} }
\end{minipage} & 2,000 & 100  & 29.35 & 12.4\\
\midrule
\begin{minipage}{0.25\textwidth}  Gaussian\\
{\small \textit{equal} }
\end{minipage} & 2,000 & 1,930  & 4.06 & $3.4 \cdot 10^{-4}$\\
\midrule
\begin{minipage}{0.25\textwidth}  Gaussian\\
{\small \textit{over parameterized} }
\end{minipage} & 2,000 & 100,000  & 1.30 & $7.4 \cdot 10^{-1}$\\
\midrule
\begin{minipage}{0.25\textwidth}  Gaussian-RF\\
{\small \textit{under parameterized} }
\end{minipage} & 2,000 & 100  & 66.38 & 3.95\\
\midrule
\begin{minipage}{0.25\textwidth}  Gaussian-RF\\
{\small \textit{equal} }
\end{minipage} & 2,000 & 1,467  & 27.15 & $6.2 \cdot 10^{-3}$ \\
\midrule
\begin{minipage}{0.25\textwidth}  Gaussian-RF\\
{\small \textit{over parameterized} }
\end{minipage} & 2,000 & 316,227  & 21.77 & $6.6 \cdot 10^{-2}$\\
\midrule
\begin{minipage}{0.25\textwidth}  MNIST\tmark[3] (all)
\end{minipage} & 60,000 & 784 & 5,562.79 & $1.1 \cdot 10^{-2}$\\
\midrule
\begin{minipage}{0.25\textwidth}  MNIST\tmark[3] RF\\
{\small large $d$}
\end{minipage} & 60,000 & 5,551 & 4,249.29 & $2.8 \cdot 10^{-1}$\\
\midrule
\begin{minipage}{0.25\textwidth}  MNIST\tmark[3] RF\\
{\small small $d$}
\end{minipage} & 60,000 & 452 & 4,564.77 & 15.09\\

\bottomrule
}

\paragraph{Real data.} 
The CIFAR-5m \citep{nakkiran2021bootstrap} example (Figures~\ref{fig:CIFAR_5M_streaming} is shown to demonstrate that large-dimensional random matrix predictions often work for large dimensional real data. Random features models were used to predict the car/plane class vector which has approximately 1 million samples. The data sets were all standardized and pre-processed to have mean $0$ and variance $1$ before applying the random features model with standardized ReLu. 

We give specific simulation/experimental details below:
\begin{itemize}
 \item \textit{CIFAR-5m streaming, Figure~\ref{fig:CIFAR_5M_streaming}}: Plots of single runs of SGD on CIFAR-5m \citep{nakkiran2021bootstrap} using the car/plane class vector (samples $n$ = 1 million, features $n_0 = 32 \times 32 \times 3$) on a random features model with standardized ReLu (see \eqref{eq:standarized_relu}). CIFAR-5m car/plane data set was standardized so that entries were mean $0$ and variance $1$.  Standard Gaussian weight matrix $\WW \in \mathbb{R}^{n_0 \times d}$ with fixed $d = 6,000$ used in the random features set-up (see Appendix~\ref{sec: motivating_applications}). Multi-pass SGD with constant learning rate $\gamma = 0.8$ applied to various sample size $n = 1000\cdot[4,6,10,20,40]$ on $\eqref{eq:rr}$ with $\delta = 0.01$. Empirical volterra solver was applied to match the multi-pass setting using the same variables. An empirical covariance $\SSigma_{\sigma}(\WW)$ computed using all 1 million samples. Streaming SGD using constant learning rate $\gamma = 0.8$ applied to the expected risk using the empirical covariance $\SSigma_{\sigma}(\WW)$. As the $\ell^2$ regularization parameter $\delta$ is hit by a factor of $n$, in the streaming setting, the regularization is set to $0.0$.  Empirical Volterra using the eigenvalues of $\SSigma_{\sigma}(\WW)$ with $\gamma = 0.8$ and $\delta = 0.0$ matched the SGD steaming setting. 
    \item \textit{Random features theory}.  
    \item \textit{ICR, Figure~\ref{fig:ICR}}: Graph of the ICR under the assumption that the normalized trace of $\nabla^2 \mathscr{L}$ is $1.0$, that is, $\tfrac{1}{n} \tr(\nabla^2 \mathscr{L}) = 1.0$. All data sets, MNIST, CIFAR-10, and CIFAR-5m are standardized (i.e., entries normalized so that mean 0.0 and variance 1.0). Largest and smallest (non-zero) eigenvalues of the feature covariance reported. For the random features set-up (RF), standard Gaussian matrix $\WW \in \mathbb{R}^{n_0 \times d}$ where $n_0$ is the underlying number of features from the data set and $d$ ranged from $10^{2.5}$ to $10^{3.9}$ was applied to the data set followed by an entry-wise activation standardized ReLu. Reported (dashed lines) are the largest and smallest eigenvalues after applying the standardized ReLu and making the normalized trace equal to 1.0. In the Gaussian set-up, the number of samples $n$ was fixed at $2000$ and $d$ ranged from $10^2$ to $10^5$; entries of $\AA$ standard Gaussians. In the random features Gaussian (Gaussian-RF), we fixed the samples $n = 2000$ and $n_0 = 100$ and varied the $d = 10^2$ to $10^{5.5}$. Largest and smallest eigenvalues of $\sigma(\XX \WW)^T \sigma(\XX \WW)$ reported after making the normalized trace $1.0$. 
    \item \textit{Gaussian linear regression streaming, Figure~\ref{fig:gaussianstreaming}}: Simple linear regression with targets from a generative model, $\bb = \AA \bbeta + \xxi$; signal $\bbeta \sim N(\bm{0}, \tfrac{1}{d} \II_d)$ and noise $\xxi \sim N(\bm{0}, \tfrac{0.04}{d} \II)$. A $(n \times 2000)$ data matrix $\AA$ with $\AA_{ij} \sim N(0, 1/2000)$ with various $n$ values (see figure). SGD with constant learning rate $\gamma = 0.8$  initialized at $\xx_0 = \bm{0}$ was applied to the linear regression problem with a regularization parameter of $0.01$, see training loss and excess risk in linear regression in Appendix~\ref{sec: motivating_applications}. In this setting, the covariance of the expected risk is explicitly given by $\II_d / d$. A new data point $\aa \sim N(0, \tfrac{1}{d} \II_d)$ and $b = \aa \bbeta + 0.2 Z$ with $Z \sim N(0,1)$ generated and the expected risk computed as $(\aa \xx_t - b)^2$ where $\xx_t$ are the iterates of SGD. Empirical volterra solver used with   grid points $\approx$ number of iterations of SGD.
    \item \textit{Gaussian linear regression concentration, Figure~\ref{fig:MSE_Gaussian}}: Simple linear regression with targets from generative model, $\bb = \AA \bbeta + \xxi$; signal $\bbeta \sim N(\bm{0}, \tfrac{1}{d}\II_d)$, noise $\xxi \sim N(\bm{0}, \tfrac{1.5^2}{n}\II_n)$. Matrix $\AA \in \mathbb{R}^{n \times d}$ is row normalized and $\tfrac{d}{n} = 0.9$ for $n = \{100, 400, 1600, 6400\}$. 10 runs of SGD with constant learning $\gamma = 0.8$ started at $\xx_0 \sim N(0, \tfrac{4}{n} \II_d)$ applied to
    the $\ell^2$-regularized least squares problem with $\delta = 0.1$, see training loss and excess risk in linear regression in Appendix~\ref{sec: motivating_applications}. 80\% confidence interval (shaded) depicted in Figure~\ref{fig:MSE_Gaussian}. Volterra equation solver used with grid points approximately the same as epochs. Expected risk computed as in Figure~\ref{fig:gaussianstreaming}. Concentration around the Volterra equation occurs as $n \text{ (or $d$)} \to \infty$ across different risk functions. 
\end{itemize}


\end{document}

%% file: macros.tex
\def\R{\mathbb{R}}
\def\Filt{\mathscr{F}}
\renewcommand\ss{{\boldsymbol s}}

\def\aa{{\boldsymbol a}}
\def\bb{{\boldsymbol b}}

\def\xx{{\boldsymbol x}}

\def\XX{{\boldsymbol X}}
\def\YY{{\boldsymbol Y}}
\def\ZZ{{\boldsymbol Z}}
\def\aa{{\boldsymbol a}}
\def\bb{{\boldsymbol b}}

\def\ee{{\boldsymbol e}}

\def\WW{{\boldsymbol W}}

\def\II{{\text{\textbf{I}}}}

\def\uu{{\boldsymbol u}}
\def\ww{{\boldsymbol w}}
\def\zz{{\boldsymbol z}}

\def\BB{{\boldsymbol B}}
\def\AA{{\boldsymbol A}}
\def\CC{{\boldsymbol C}}

\def\MM{{\boldsymbol M}}

\def\PP{{\boldsymbol P}}
\def\TT{{\boldsymbol T}}

\def\QQ{{\boldsymbol Q}}

\def\HH{{\boldsymbol H}}

\def\xxi{  {\boldsymbol \xi} } 
\def\SSigma{{\boldsymbol \Sigma}}

\def\YY{ {\boldsymbol Y} }

\def\eeta{{\boldsymbol \eta}}

\def\bbeta{{\boldsymbol \beta}}

\def\dif{\mathop{}\!\mathrm{d}}

\def\EE{{\mathbb E}\,}

\DeclareMathOperator{\tr}{tr}
\def\defas{\stackrel{\text{def}}{=}}

\DeclareDocumentCommand{\Prto} {o} {
  \IfNoValueTF {#1}
  {\overset{\Pr}{\longrightarrow}}
  { \xrightarrow[ #1 \to \infty]{\Pr }}
}

\DeclareDocumentCommand{\Asto} {o} {
  \IfNoValueTF {#1}
  {\overset{\text{\rm a.s.}}{\longrightarrow}}
  { \xrightarrow[ #1 \to \infty]{\text{\rm a.s.} }}
}

\DeclareDocumentCommand{\law} {o} {
  \IfNoValueTF {#1}
  {\overset{\text{law}}{=}}
  { \xrightarrow[ #1 \to \infty]{\Pr }}
}

\DeclareMathOperator{\Exp}{\mathbb{E}}

\newcommand{\vertiii}[1]{{\left\vert\kern-0.25ex\left\vert\kern-0.25ex\left\vert #1 
    \right\vert\kern-0.25ex\right\vert\kern-0.25ex\right\vert}}
    
\newcommand{\gf}{\bm{\mathscr{X}}^{\text{gf}}}

\newcommand{\sgf}{\bm{\mathscr{X}}^{\text{s-gf}}}

\newcommand{\sgd}{\bm{\xx}^{\text{sgd}}}

\newcommand{\mgd}{\bm{x}^{\text{m-gd}}}